\documentclass[lettersize,journal]{IEEEtran}
\usepackage{array}
\usepackage{subcaption}
\usepackage{textcomp}
\usepackage{stfloats}
\usepackage{verbatim}
\usepackage{graphicx}
\usepackage{cite}
\usepackage{xcolor}
\usepackage{amsmath, amssymb, amsfonts, amsthm}
\usepackage{mathtools}
\usepackage{booktabs}
\usepackage{hyperref}
\usepackage{url}
\usepackage{enumitem}
\usepackage{cases}
\usepackage[ruled]{algorithm2e}
\usepackage{multirow}
\theoremstyle{definition}
\newtheorem{definition}{Definition}
\theoremstyle{plain}
\newtheorem{theorem}{Theorem}
\newtheorem{corollary}{Corollary}
\newtheorem{lemma}{Lemma}

\newtheorem{assumption}{Assumption}
\newtheorem{proposition}{Proposition}

\allowdisplaybreaks

\newtheorem{remark}{Remark}
\title{\LARGE \bf Sparse Network Inference under Imperfect Detection \\ and its Application to Ecological Networks}

\author{
  Aoran Zhang, Tianyao Wei, Maria J. Guerrero, C\'esar A. Uribe
  \thanks{This work was funded by the Ken Kennedy Institute, Fulbright, and the National Science Foundation, United States, under Grant \#2443064. }
\thanks{A.Z. and C.A.U. (\textit{\{az73,cauribe\}@rice.edu}) are with the Department of Electrical and Computer Engineering and the Ken Kennedy Institute. T.W. (\textit{\{cw158\}@rice.edu}) is with the Department of Computational and Applied Mathematics at Rice University. M.J.G. (\textit{\{mg160\}@rice.edu}) is with SISTEMIC, Facultad de Ingenier\'ia, Universidad de Antioquia and the Department of Electrical and Computer Engineering at Rice University. }
}

\begin{document}
\maketitle
\vspace{-0.2in}

\begin{abstract}
    Recovering latent structure from count data has received considerable attention in network inference, particularly when one seeks both cross-group interactions and within-group similarity patterns in bipartite networks, which are widely used in ecology research. Such networks are often sparse and prone to imperfect detection. Existing models mainly focus on interaction recovery, while the induced similarity graphs are much less studied. Moreover, \textcolor{black}{existing models for under-counted data place sparsity, if at all, on the detection model rather than on these graphs \cite{nguyen2025under}}. To address these issues, we propose a nonnegative low-rank factorization framework that jointly estimates detection probabilities and promotes sparse within- and cross-group graphs through $\ell_{1/2}$ regularization. For the resulting nonconvex, nonsmooth problem, we develop an ADMM-based algorithm with adaptive penalization, and establish asymptotic feasibility and KKT stationarity of cluster points under stated assumptions. \textcolor{black}{We also give sufficient conditions for exact recovery of within- and cross-group graphs. On synthetic data with up to 20\% missing entries, the method recovers the graphs and detection probabilities with competitive accuracy. Across three ecological datasets, it matches the baseline held-out predictions while estimating detection probabilities and within- and cross-group graphs.}
\end{abstract}


\section{Introduction}
Latent recovery of connectivity structure from biased and partial observations is a recurring problem in signal processing and network inference~\cite{Dong2016, Mateos2019}. In many scenarios, the object of interest is a bipartite interaction matrix with entries encoding frequencies, affinities, or event intensities between two classes of entities. \textcolor{black}{When the data are observed via noisy counting processes, inference goes beyond standard matrix completion tasks, as one must jointly account for latent low-rank structure, measurement distortion, and potential structural constraints on the recovered network.} This setting is inherent in sensing and monitoring applications~\cite{shafique2022imputing, nakashima2022double}, where observations, such as counts, are obtained via an imperfect sampling process.

In this paper, we are interested in ecological interaction networks describing how species associate with locations and how environments shape biodiversity patterns~\cite{Dawson2026, Suresh2025}. These systems can naturally be represented as bipartite graphs, in which one node set corresponds to species and the other to geographic areas~\cite{Guerrero2025, Duarte2021}. The intensity of interaction is recorded as the frequency of observations in field surveys. However, observations are limited by imperfect detection: even when a species is present at one location, its presence may go unrecorded due to factors such as limited sampling times, occlusions, or misidentification~\cite{jordano2016sampling}. 
The goal here is to denoise the measurements \textcolor{black}{while recovering interpretable structures: cross-group connectivity between locations and species, as well as within-group similarity patterns among locations and species. }
This naturally motivates a structured latent-factor model in the presence of imperfect observations.  
Imperfect detection must be modeled explicitly, or otherwise, zeros and small counts may be incorrectly interpreted as weak or absent interactions rather than missing recordings.

Prior work addressed this issue by modeling latent interaction intensity through nonnegative matrix factorization and estimating detection probabilities separately in an alternating scheme~\cite{fu2019link}. \textcolor{black}{Its extensions model the detection process more flexibly: \cite{ibrahim2023under} learns detection probabilities with a neural network, and \cite{nguyen2025under} further imposes sparsity on the detection matrix’s deviations from a low-rank pattern. However, none of them places structured sparsity on the induced within-group similarity and cross-group connectivity graphs, which are our estimation targets.}
In bipartite network recovery problems, the within-group similarity structures \textcolor{black}{play an important role alongside cross-group interactions} for interpretation and downstream analysis \textcolor{black}{\cite{gao2020learning, liu2016neighborhood, melamed2014community}}. In ecological applications, only a subset of latent relationships is expected to be active, reflecting selective co-occurrence, habitat specialization, or structured overlap in detectability. For instance, monitoring sites within the same habitat type may have similar acoustic community compositions because they share environmental conditions, whereas sites separated by land-use barriers or fragmentation may host quite different species assemblages even under the same cover classification~\cite{Guerrero2025}. This motivates us to impose structured sparsity on both the cross-group connectivity and within-group similarity matrices.

\textcolor{black}{The bilinear intensity model and nonconvex graph penalties yield a nonconvex, nonsmooth estimation problem. The penalties couple the factors through quadratic products, leading to a block-structured optimization problem with nonlinear, nonseparable constraints.
To solve this,} we formulate latent connectivity recovery under imperfect detection as a structured, sparse, nonnegative, low-rank estimation problem. 
The contributions of this paper are as follows:
\begin{itemize}[leftmargin=2.0em,left=0pt,itemsep=1pt, parsep=-1pt, topsep=-0pt, partopsep=-1pt]
\item We propose a structured latent-factor model for bipartite count data under imperfect detection \textcolor{black}{with feature-dependent detection and $\ell_{1/2}$ regularization on the induced within-group similarity and cross-group connectivity graphs.}
\item \textcolor{black}{We analyze identifiability and recoverability of the three graphs, giving sufficient conditions for exact edge-set recovery. For linear detection, we establish a finite-sample mean-error bound for global minimizers.}
\item \textcolor{black}{We develop a model-specific ADMM-style splitting scheme with exact half-thresholding graph updates. Its detection block accommodates linear and smooth nonlinear models without changing the factor or graph updates. Under the stated regularity and inexactness assumptions, we establish asymptotic feasibility and KKT stationarity of cluster points.}
\item \textcolor{black}{We evaluate graph and detection recovery on synthetic networks and held-out prediction on three ecological datasets, with comparisons against existing baselines.}
\end{itemize}

The rest of this paper is organized as follows: Section II reviews related work. 
\textcolor{black}{Sections III--IV present the model and recoverability analysis.
Sections IV--V present the optimization algorithm and the convergence analysis.} Section VI reports experimental results on synthetic and real datasets.

\section{Related Work}
\paragraph{N-mixture Model} 
Ecological interaction networks can be represented as bipartite networks in which interactions occur across two groups~\cite{jordano2003invariant, bascompte2007networks, dallas2017predicting, Guerrero2025}. In such systems, observed interaction counts arise from both underlying ecological processes and detection mechanisms.
N-mixture models explicitly separate latent ecological processes from the stochastic detection effects.~\cite{royle2004n, joseph2009modeling,
mackenzie2002estimating, kellner2014accounting}. 
A recent work combined the nonnegative factorization of interaction intensities with feature-dependent linear detection to extend imperfect-detection models to ecological interaction networks~\cite{fu2019link}. 
\textcolor{black}{This model has been extended to neural network detection functions \cite{ibrahim2023under}, and to data without detection features, where the detection matrix is modeled as a low-rank plus a sparse component \cite{nguyen2025under}.}

While most ecological network studies focus on cross-group interactions, less attention has been given to the structural relationships within each group~\cite{Guerrero2025}. Within-group relationships also support bipartite representation learning and community analysis~\cite{gao2020learning, liu2016neighborhood, melamed2014community}.
In practice, underlying within- and cross-group relationships can be sparse, with many potential links absent or unobserved \cite{jordano2016sampling, jordano2003invariant, terry2020finding, barker2018reliability}. Therefore, we aim to recover both of these underlying relations using structured latent embeddings with sparsity regularization to promote interpretable and ecologically plausible network recovery.

\paragraph{Nonnegative Matrix Factorization}
Low-rank factorization represents target matrices as low-dimensional embeddings via alternating minimization or projected gradient methods~\cite{koren2009matrix, lee2000algorithms, haeffele2014structured, jain2013low, huang2016flexible, park2017non, park2018finding}. 
Low-rank-plus-sparse models, including robust PCA, decompose the target matrix as $\mathbf X = \mathbf L + \mathbf{S}$ with $\mathbf{L}$ low-rank and $\mathbf{S}$ sparse~\cite{candes2011robust, bertsimas2023sparse, zhang2018unified, gu2016low, shang2014recovering}. This places sparsity on a separate data component instead of directly on the factor graphs. Sparse NMF promotes sparsity by imposing sparsity constraints on factor vectors through $\ell_1/\ell_2$ ratios. However, dense pairwise graphs can still be induced from sparse factor vectors.   

The work closest to ours is sparse reduced-rank regression and structured factorization, where $\mathbf X$ is decomposed directly into $\mathbf U\mathbf V^\top$ and the sparsity regularization is imposed on $\mathbf U\mathbf V^\top$~\cite{dubois2019fast, haeffele2019structured}. 
An orthogonality constraint is enforced ~\cite{dubois2019fast}, and the sparsity constraint is convex~\cite{dubois2019fast} or separable~\cite{haeffele2019structured}.
Different from these approaches, we impose $\ell_{1/2}$ sparsity penalties on quadratic latent factor products. These regularizers are nonconvex, nonseparable, and nonsmooth across factor blocks, motivating the need for an ADMM-based framework to isolate the nonsmooth penalties.

\paragraph{ADMM \textcolor{black}{in Nonconvex Settings}} 
The analysis in~\cite{hong2016admm} establishes the convergence of ADMM to stationary points under sufficiently large fixed penalties for nonconvex consensus and sharing formulations. In~\cite{wang2019global}, the authors extended this line to a broader class of nonconvex, possibly nonsmooth objectives with coupled linear equality constraints. \textcolor{black}{These results do not directly cover our problem since our splitting constraints are nonlinear quadratic equalities, not affine couplings.} For general nonlinearly coupled constraints, the penalty dual decomposition provides a broader primal--dual approach with convergence to KKT points~\cite{shi2020penalty}. Related ADMM variants include inexact linearized ADMM~\cite{el2025convergence} and IPDS-ADMM~\cite{yuanadmm}. \textcolor{black}{Our goal differs from these frameworks: we specialize the splitting to isolate product-graph penalties, so that the auxiliary blocks admit exact half-thresholding updates and factor updates reduce to smooth, nonnegative subproblems.}

\section{Model and Problem Statement}\label{sec:model}
We model the observed counts as the interactions between two groups in a bipartite network through a Poisson $N$-mixture process~\cite{fu2019link}. \textcolor{black}{For each member of one group $i \in \{1,\ldots,I\}$ and the other group $j \in \{1,\ldots,J\}$.} We denote $\mathbf u_i,\mathbf v_j \ge \mathbf 0$ as the $i$-th row and $j$-th row of the latent factor matrices $\mathbf U \in \mathbb R^{I \times F}$ and $\mathbf V \in \mathbb R^{J \times F}$ where $\mathbf U$ and $\mathbf V$ encode the latent embeddings. Here, $F$ denotes the latent rank, which induces a low-rank structure on the interaction-intensity matrix. The latent interaction intensity between $i$ and $j$ is then $\lambda_{ij} = \mathbf  u_i^\top \mathbf v_j$.

\textcolor{black}{\paragraph{Model Formulation} Let $\mathbf{N}_{ij}$ be the true number of interactions between $i$ and $j$ in one recording and let each event be detected independently with probability $\mathbf{P}_{ij}$, which depends on factors such as weather or background noise. Conditional on the model parameters and features, the latent counts and their detection processes are independent across cells. Then
\begin{align*}
\mathbf N_{ij} \sim \mathrm{Poisson}(\lambda_{ij}), \quad 
\mathbf Y_{ij} \big| \mathbf N_{ij} &\sim \mathrm{Binomial}\!\left(\mathbf N_{ij},\mathbf{P}_{ij}\right).  
\end{align*}
As in the standard N-mixture model~\cite{royle2004n, fu2019link}, we allow imperfect detections, i.e., missing entries but no false positives, so $\mathbf{Y}_{ij} \leq \mathbf{N}_{ij}$.
Summing out the latent count, 
\begin{multline*}
\mathbb P(\mathbf{Y}_{ij}=y)
= \sum_{n=y}^{\infty} \frac{e^{-\lambda_{ij}}\lambda_{ij}^{n}}{n!}\binom{n}{y}\mathbf{P}_{ij}^{y}(1-\mathbf{P}_{ij})^{n-y}\\
= \frac{e^{-\lambda_{ij}}(\mathbf{P}_{ij}\lambda_{ij})^{y}}{y!}\sum_{n=y}^{\infty}\frac{\big((1-\mathbf{P}_{ij})\lambda_{ij}\big)^{n-y}}{(n-y)!}\\
= \frac{e^{-\mathbf{P}_{ij}\lambda_{ij}}(\mathbf{P}_{ij}\lambda_{ij})^{y}}{y!}.
\end{multline*}
Hence $\mathbf Y_{ij} \sim \mathrm{Poisson}(\mathbf{P}_{ij}\lambda_{ij})$.}

\textcolor{black}{\paragraph{Replicates and Sampling Effort}
Consider $M$ replicates with known efforts $n^{(1)}, \ldots, n^{(M)} > 0$. Conditional on the model parameters and features, the replicates are independent, with a new latent Poisson count for each cell in each replicate. We assume that the intensities are proportional across replicates, i.e., $\lambda_{ij}^{(m)} = n^{(m)}\lambda_{ij}$, and the detection is shared, i.e., $\mathbf{Y}_{ij}^{(m)} \sim \mathrm{Poisson}(n^{(m)}\mathbf{P}_{ij}\lambda_{ij})$. Then for $E := \sum_{m = 1}^M n^{(m)}$, the summed count 
\[
\mathbf{Y}_{ij} = \sum_{m = 1}^M \mathbf{Y}_{ij}^{(m)} \sim \mathrm{Poisson}(E\mathbf{P}_{ij}\lambda_{ij}).
\]
}

\textcolor{black}{
\paragraph{Detection Probability} The detection probabilities form a matrix $\mathbf P \in [0,1]^{I\times J}$ and depend on known features of each pair,
\[
\mathbf{P}_{ij} = g(\mathbf Z_{ij}; \boldsymbol\theta), \quad 0 \leq \mathbf{P}_{ij} \leq 1,
\]
where $\mathbf Z_{ij} \in \mathbb R^{R}$ is a known feature vector, $\boldsymbol{\theta} \in \Theta$ is an unknown parameter, and $g$ is continuously differentiable in $\boldsymbol{\theta}$. The convergence analysis (Section \ref{sec:converg:analysis}) includes the general form of the detection model. $\mathbf Z \in \mathbb R^{(IJ)\times R}$ encodes domain knowledge of characteristics that influence detectability. For example, in sonotype-specific acoustic data, this includes the dominant frequency, the minimum and maximum vocalization frequencies, and the vocalization duration. We use the linear model of \cite{fu2019link},
\[
g(\mathbf Z_{ij};\boldsymbol{\alpha}) = \mathbf Z_{ij}^\top \boldsymbol{\alpha}, \quad \Theta = \{\boldsymbol{\alpha} \in \mathbb R^R: \mathbf 0 \leq \mathbf Z \boldsymbol{\alpha} \leq \mathbf 1\},
\]
where we define $\mathrm{vec}(\mathbf P) := \mathbf Z \boldsymbol{\alpha}$ as pairwise detection probability.
A nonlinear alternative could be a neural network with a sigmoid output \cite{ibrahim2023under}, which, as we show in Section \ref{app:nonlinear}, helps mainly when detection is strongly nonlinear.
}

\textcolor{black}{Define $\Omega \subseteq [I]\times [J]$
be the set of observed entries, and entries outside $\Omega$ are structurally missing. Up to additive constants, the log-likelihood function is
\begin{multline*}
L(\mathbf U, \mathbf V, \boldsymbol{\alpha}) = \sum_{(i,j)\in\Omega}\Big[\mathbf{Y}_{ij}\log\big(g(\mathbf Z_{ij};\boldsymbol\alpha)\,\mathbf u_i^\top \mathbf v_j\big) \\- E\,g(\mathbf Z_{ij};\boldsymbol\alpha)\,\mathbf u_i^\top \mathbf v_j\Big].
\end{multline*}}

\paragraph{Sparsity} 
To promote sparse structure in the latent patterns, we impose penalties on the norms of $\mathbf U\mathbf U^\top$, $\mathbf V\mathbf V^\top$, and $\mathbf U\mathbf V^\top$.  While the $\ell_1$ penalty is a standard convex surrogate for sparsity, it introduces stronger amplitude bias on large coefficients and can overly attenuate significant latent relations, leading to oversmoothed graph estimates. Though $\ell_0$ enforces stronger sparsity while exerting less bias on large nonzero entries, it leads to a combinatorial problem whose complexity grows rapidly in the present factorized and constrained setting. The $\ell_{1/2}$ penalty offers a favorable compromise: it preserves the sparsity-inducing behavior of nonconvex regularization, yet still admits an explicit half-thresholding characterization for the corresponding proximal subproblem. Thus, we impose $\ell_{1/2}$ quasi-norm penalties on these matrices: 
\vspace{-0.3em}
\[
\textcolor{black}{\lambda_{UU}\|\mathbf U\mathbf U^\top\|_{1/2}
+\lambda_{UV}\|\mathbf U\mathbf V^\top\|_{1/2}
+\lambda_{VV}\|\mathbf V\mathbf V^\top\|_{1/2}},
\]
\textcolor{black}{where $\|\mathbf A\|_{1/2} := \sum_{ij}|a_{ij}|^{1/2}$.}
Denoting 
$\mathbf M_{UU} := \mathbf U\mathbf U^{\top}, \  \mathbf M_{UV} := \mathbf U\mathbf V^{\top}, \ \mathbf M_{VV}:= \mathbf V\mathbf V^{\top}$, the optimization problem becomes:
\textcolor{black}{
    \begin{align}
    \min_{\mathbf U, \mathbf V, \boldsymbol{\alpha}} &
    -L(\mathbf U, \mathbf V, \boldsymbol{\alpha})
     + \sum_{X \in \{UU, VV, UV\}} \lambda_{X}\lVert \mathbf M_X \rVert_{1/2} \label{eq:overall_problem}\\
    \text{subject} & \text{ to }
    \mathbf  U \ge \mathbf  0, \  \mathbf  V \ge \mathbf  0, \ 0 \le g(\mathbf Z_{ij}; \boldsymbol{\alpha})  \le 1, \ \forall i,j. \nonumber
\end{align}
}
which is non-convex over \textcolor{black}{$(\mathbf U,\mathbf V,\boldsymbol{\alpha})$} with a block structure, motivating an alternating strategy. 
For fixed \textcolor{black}{$(\mathbf U,\mathbf V)$}, the $\boldsymbol{\alpha}$-subproblem is handled as in~\cite{fu2019link}. For fixed $\boldsymbol{\alpha}$, we solve the \textcolor{black}{$(\mathbf U,\mathbf V)$} block through a split augmented-Lagrangian scheme with auxiliary variables for the nonsmooth sparsity terms.

\section{Recoverability of the Graphs}\label{sec:ident}
{\color{black}

For any $a_U,a_V,a_\alpha>0$ with $a_Ua_Va_\alpha=1$ and $a_\alpha\mathbf Z\boldsymbol\alpha\le\mathbf 1$, the parameters $(a_U\mathbf U,a_V\mathbf V,a_\alpha\boldsymbol\alpha)$ and $(\mathbf U,\mathbf V,\boldsymbol\alpha)$ give $\mathbf Y$ the same distribution, as in any NMF-based model. The scale of an estimate is therefore set by the initialization and the regularization, not by the data, and we evaluate only scale-free quantities (Section~\ref{sec:experiments}). In contrast, the edge set $\mathcal E(\mathbf G)=\{(i,j):G_{ij}\neq0\}$ of each graph $\mathbf G$ is invariant to this rescaling.

 Write $\boldsymbol\zeta=(\mathbf U,\mathbf V,\boldsymbol\alpha)$, $\mu_{ij}(\boldsymbol\zeta)=(\mathbf Z_{ij}^\top\boldsymbol\alpha)\,\mathbf u_i^\top\mathbf v_j$ for the mean per unit effort, and $\mathbf M_X(\boldsymbol\zeta)$, $X\in\{UU,VV,UV\}$, for the three graphs. The truth $\boldsymbol\zeta_0$ has intensity $\boldsymbol\Lambda_0=\mathbf U_0\mathbf V_0^\top$, $\mathbf P_0>0$, $\boldsymbol\mu_0=\mathbf P_0\odot\boldsymbol\Lambda_0$, and $\mathcal E_0=\mathcal E(\boldsymbol\mu_0)$. As in~\cite{ibrahim2023under}, we study a global minimizer over a bounded class,
\begin{equation}\label{eq:mle}
\widehat{\boldsymbol\zeta}\in\arg\min_{\boldsymbol\zeta\in\mathcal C}\ -L_{\delta_E}(\boldsymbol\zeta)+\mathcal R(\boldsymbol\zeta),
\end{equation}
where $\mathcal C=\{\boldsymbol\zeta:\mathbf 0\le\mathbf U,\mathbf V\le B,\ \boldsymbol\alpha\in\Theta,\ \mathbf Z\boldsymbol\alpha\ge p_{\min}\}$ restricts the feasible set of~\eqref{eq:overall_problem} by an upper bound $B$ on the factor entries and a lower bound $p_{\min}\in(0,1]$ on the detection probabilities, $\mathcal R(\boldsymbol\zeta)=\sum_X\lambda_X\|\mathbf M_X(\boldsymbol\zeta)\|_{1/2}$, and $L_{\delta_E}$ is $L$ with each $\mu_{ij}$ floored to $\mu_{ij}+\delta_E$, $\delta_E=1/(E|\Omega|)$, so that the likelihood stays finite where $\mu_{0,ij}=0$. We write $\widehat\mu_{ij}=\mu_{ij}(\widehat{\boldsymbol\zeta})$ and $d=F(I+J)+R$ for the number of parameters.

\begin{definition}\label{def:ident}
The edge sets are \emph{identified} on $\mathcal C$ if every $\boldsymbol\zeta\in\mathcal C$ that fits the mean exactly, i.e., $\mu_{ij}(\boldsymbol\zeta)=\mu_{0,ij}$ for all $(i,j)\in\Omega$, has $\mathcal E(\mathbf M_X(\boldsymbol\zeta))=\mathcal E(\mathbf M_X(\boldsymbol\zeta_0))$ for all $X$.
\end{definition}
Under identification, we introduce two margins measuring how hard recovery is (formal definitions in~\eqref{eq:margins}): $w_{\min}$ is the weakest true edge weight over all exact fits, and $\kappa$ is the smallest loss of fit of a parameter whose graphs differ from those of every exact fit by at least $w_{\min}/3$ in some entry.

\begin{theorem}\label{thm:finite}
Let $\boldsymbol\zeta_0\in\mathcal C$,  $\mathbf Z$ have full column rank,  $t>0$, and  $\varepsilon_E=O\big((d\log(E|\Omega|)+t+\mathcal R(\boldsymbol\zeta_0))/(E|\Omega|)\big)$ be a bound on the average squared error between $\sqrt{\widehat\mu_{ij}+\delta_E}$ and $\sqrt{\mu_{0,ij}}$ over the observed entries (exact value in~\eqref{eq:epsE}). With probability at least $1-2e^{-t}$, this bound holds for every global minimizer $\widehat{\boldsymbol\zeta}$ of~\eqref{eq:mle}, and $\widehat{\boldsymbol\zeta}$ satisfies the following.
\begin{enumerate}[label=(\alph*),leftmargin=*]
\item Let $\lambda_{0,ij}\ge\beta>0$ on  $\mathcal E_0\cap\Omega$, and let $\widehat{\mathcal E}=\{(i,j)\in\Omega:\widehat\mu_{ij}+\delta_E\ge p_{\min}\beta/2\}$. The fraction of misclassified cells satisfies
\[
\frac{|\widehat{\mathcal E}\,\triangle\,(\mathcal E_0\cap\Omega)|}{|\Omega|}\le\frac{\varepsilon_E}{c_0p_{\min}\beta}, \quad c_0=(1-1/\sqrt2)^2,
\]
where $\triangle$ denotes the symmetric difference. If $|\Omega|\varepsilon_E<c_0p_{\min}\beta$, then $\widehat{\mathcal E}=\mathcal E_0\cap\Omega$, and $\widehat\mu$ ranks every link in $\Omega$ above every non-link in $\Omega$.
\item If the edge sets are identified, then $w_{\min}>0$ and $\kappa>0$, and if $3\varepsilon_E<\kappa$, thresholding each $\mathbf M_X(\widehat{\boldsymbol\zeta})$ at $w_{\min}/2$ recovers all three edge sets exactly.
\end{enumerate}
\end{theorem}

The proof is deferred to Appendix \ref{app:proofs:ident}. The mean-error bound is proved for linear detection, while the
conversion in (a) applies to any detection model with
$\mathbf P_0\ge p_{\min}>0$ and a corresponding mean-error bound. Positive detection preserves $\mathcal E(\mathbf U\mathbf V^\top)\cap\Omega=\mathcal E_0\cap\Omega$, so (a) needs no additional identification assumption.
For fixed $t$ with full observation, if  $\mathcal R(\boldsymbol\zeta_0)=O(d\log(E|\Omega|))$, the misclassified fraction is $O\big(d\log(E|\Omega|)/(IJ\,E\,p_{\min}\beta)\big)$, i.e., it decays as $\log E/E$. 

The within-group graphs are not identifiable in general, since an unrestricted detection model, where each detection probability $\mathbf P_{ij}\in(0,1]$, can be chosen freely, can absorb any entrywise rescaling of $\mathbf U\mathbf V^\top$. The linear model shares its coefficients across all pairs and can rule out this ambiguity under the feature-density assumption of Proposition~\ref{prop:generic} (Appendix \ref{app:ident}). With full observation, it identifies all three edge sets almost surely when the link-count condition~\eqref{eq:count} holds, $\mathrm{rank}(\boldsymbol\Lambda_0)=F$, and the nonnegative factorizations of $\boldsymbol\Lambda_0$ share their within-group edge sets, e.g., when both factors are sufficiently scattered~\cite{huang2014nmf,fu2019nmf} (Appendix~\ref{app:ident}). With missing entries, identifying the full graphs also requires that all feasible factorizations of admissible completions induce the same three edge sets.
\begin{remark}
The Poisson bound of~\cite{soni2016noisy} requires all means greater than $0$, and that of~\cite{ibrahim2023under} requires every factor entry to be greater than $0$. Both make the graphs complete. The bias terms of~\cite{ibrahim2023under} also come from approximating detection as constant on clusters of similar features. Theorem~\ref{thm:finite} allows zero means and factor entries, has no bias term, and allows any fixed $\Omega$ rather than uniform sampling. Like~\cite{ibrahim2023under}, it concerns a global minimizer rather than the stationary points of Section~\ref{sec:converg:analysis}, and $w_{\min}$ and $\kappa$ are not explicit. It also claims no gain from sparsity: the penalty enters only through $\mathcal R(\boldsymbol\zeta_0)$ and serves to select the scale and to regularize in finite samples.
\end{remark}
}

\section{Optimization Algorithm}

We exploit the structure in Problem~\eqref{eq:overall_problem} and propose the ADMM-based Algorithm~\ref{alg:merged_soundscape_code}. We discuss each subproblem in the following subsections, and its convergence analysis is presented in Section~\ref{sec:converg:analysis}.

\begin{algorithm}[t]
\caption{\textcolor{black}{ADMM for the Proposed Model \eqref{eq:overall_problem}}}
\label{alg:merged_soundscape_code}

\KwIn{$\mathbf Y \in \mathbb R_+^{I \times J}$, \textcolor{black}{$E > 0$},
 \textcolor{black}{$\mathbf Z\in\mathbb{R}^{(IJ)\times R}$},
 $F \in \mathbb N$, $\gamma > 1$,   $p_{0} \in (0,1)$, \textcolor{black}{$\delta > 0$}, $\epsilon_X^k >0$, $\lambda_{X},\rho_X > 0$ for $X \in \{UU, UV, VV\}$\;} 

\textbf{Initialize} $\mathbf U^0 \in \mathbb R^{I \times F},\mathbf V^0 \in \mathbb R^{J \times F}$ via Appendix~\ref{sec:init}, $\boldsymbol{\alpha}^0, \mathbf A_{X}^0 , \mathbf W_{X}^0$ for $X \in \{UU, UV, VV\}$\;

\For{$k = 1,2,\ldots$}{
    \vspace{-0.2cm}
    \textcolor{black}{
    \begin{align*}
    & \hspace{-1.25em}\boldsymbol{\alpha}
    \approx_{\xi_k} 
    \arg\min_{\boldsymbol{\alpha}}
    \sum_{ij}
    -\Big[
    \mathbf Y_{ij}\log g(\mathbf Z_{ij};\boldsymbol{\alpha})
    - Eg(\mathbf Z_{ij};\boldsymbol{\alpha})\,\mathbf{u}_{i}^\top\mathbf{v}_j
    \Big].\\
    & \text{subject to } 0 \leq  \textcolor{black}{g(\mathbf Z_{ij};\boldsymbol{\alpha})} \le 1.
    \end{align*}}
    \hspace{-0.1cm} Fill missing entries $(i,j) \notin {\Omega}$ with \textcolor{black}{\[Eg(\mathbf Z_{ij};{\boldsymbol\theta})\max(\mathbf u_i^\top\mathbf v_j,\delta).\]}
    \vspace{-0.6cm}
    \begin{align*} 
     & \mathbf U  \approx_{\varepsilon_k}   \arg\min_{\mathbf U \ge 0}   f(\mathbf U,\mathbf V) + \sum_{X} \frac{\rho_{X}}{2}\|\mathbf M_X - \mathbf A_{X} + \mathbf W_{X}\|^2.\\
 & \mathbf V  \approx_{\varepsilon_k} 
    \arg\min_{\mathbf V \ge 0}
    f(\mathbf U,\mathbf V)+ \sum_{X} \frac{\rho_{X}}{2}\|\mathbf M_X {-} \mathbf A_{X} + \mathbf W_{X}\|^2.\\
 & \mathbf A_X  \gets \arg\min_{\mathbf A_X}
\lambda_X \|\mathbf A_X\|_{\frac 1 2} {+}
\frac{\rho_X}{2}
\|\mathbf M_X - \mathbf A_X +\mathbf  W_X\|_F^2 . 
\end{align*}
\eIf{$\|\mathbf M_X - \mathbf A_X\|_F \le \epsilon_X^k$}{ 
        $\mathbf W_X = \mathbf W_X + \mathbf M_X - \mathbf A_X$\;
    }{
        $\rho_X = \gamma \rho_X$,
        $\mathbf W_X = \frac{1}{\gamma}\left(\mathbf M_X - \mathbf A_X + \mathbf W_X \right)$.  
    }
}

\Return{$\mathbf U,\mathbf V, \boldsymbol{\alpha}$}\;
\end{algorithm}

\begin{remark}
    Details on the approximation for $\alpha$-subproblem will be referred to in Assumption~\ref{assum:alpha_inexact} and on the approximation for the $\mathbf U/\mathbf V$-subproblem will be referred to in Assumption~\ref{assum:err_con}.
\end{remark}

\subsection{\texorpdfstring{$\boldsymbol{\alpha}$}{alpha}-update}\label{sec:alpha-update}
Following~\cite{fu2019link}, the $\boldsymbol{\alpha}$-subproblem is solved by introducing an auxiliary variable \textcolor{black}{$\tilde{\mathbf P} \in \mathbb R^{I \times J}$ such that $\tilde{\mathbf{P}}_{ij} = \mathbf{Z}_{ij}^\top\boldsymbol{\alpha}$ and $\boldsymbol{\omega}$ as the scaled dual variable associated with $\mathbf Z \boldsymbol{\alpha}  = \mathrm{vec}(\mathbf{\tilde{P}})$.} The subproblem can be written as 
\begin{equation}\label{eq:alpha_subprob}
    \begin{aligned}
    &\min_{\boldsymbol{\alpha}, \mathbf{\textcolor{black}{\tilde P}}}  \textcolor{black}{\sum_{ij} - \left[\mathbf{Y}_{ij}\log \mathbf{\tilde P}_{ij} -  E \tilde{\mathbf P}_{ij}\lambda_{ij} \right]} \\
    & \quad \text{subject to}  \quad 0 \leq \textcolor{black}{\mathbf{\tilde P}}_{ij} \leq 1, \textcolor{black}{\mathbf Z_{ij}^\top\boldsymbol{\alpha}  =  \mathbf{\tilde P}_{ij}}.
    \end{aligned}
    \end{equation}
The subproblem objective is \textcolor{black}{$\phi(\mathbf{\tilde P};\lambda^k) := \sum_{ij}(
- \mathbf{Y}_{i,j}\log  \tilde{\mathbf P}_{ij} + E\,\lambda^k_{ij}\, \tilde{\mathbf P}_{ij})$} with the augmented Lagrangian 
\[
    \mathcal{L}(\boldsymbol{\alpha}, \mathbf{\textcolor{black}{\tilde P}}, \boldsymbol{\omega}) = \phi(\mathbf{\textcolor{black}{\tilde P}};\lambda^k) + \frac{\rho}{2}\lVert \mathrm{vec}(\mathbf{\textcolor{black}{\tilde P}}) - \mathbf Z\boldsymbol{\alpha} + \boldsymbol{\omega} \rVert^2.
\]
where $\rho$ is the penalty parameter. 
\textcolor{black}{Define $\bar{\mathbf P}:= \mathrm{vec}(\mathbf{\tilde P})$.}
The closed-form update for $\mathbf{\tilde P}$, $\boldsymbol{\alpha}$, and  $\boldsymbol{\omega}$ is obtained with~\cite{fu2019link}:
\textcolor{black}{\begin{align*}
    \mathbf{\tilde P} & \leftarrow \left[ \Big( {(\rho \bar{\mathbf P} - E \lambda) + \sqrt{4\rho \mathbf{Y} +\left(\rho \bar{\mathbf P} - E \lambda\right)^2}}\Big)/{(2\rho)} \right]_{[0,1]},\\
    \boldsymbol{\alpha} & \leftarrow  \mathbf Z^\dagger(\mathrm{vec}(\mathbf{\tilde P}) + \boldsymbol{\omega}), \ \ \  \ \ 
    \boldsymbol{\omega}  \leftarrow \boldsymbol{\omega} + \mathrm{vec}(\mathbf{\tilde P}) - \mathbf Z\boldsymbol{\alpha},
\end{align*}
}

where $\mathbf Z^\dagger$ is the pseudo-inverse and the update is repeated until convergence by~\cite{fu2019link}. The optimality of $\boldsymbol\alpha$-update is guaranteed since the $\boldsymbol{\alpha}$-subproblem is convex~\cite{boyd2011distributed, fu2019link}.

\paragraph*{General Detection} \textcolor{black}{For a nonlinear $g$, we update $\boldsymbol\theta$ directly without a splitting variable. Starting from $\boldsymbol\theta_k$, we approximately minimize $\phi(g(\mathbf Z;\boldsymbol\theta);\boldsymbol\lambda_k)+r(\boldsymbol\theta)$ until the stationarity condition in Assumption~\ref{assum:alpha_inexact} holds.}

\subsection{\texorpdfstring{$\mathbf U$, $\mathbf V$}{U, V}-update}
\label{sec:uv_update}
Fix $\boldsymbol{\alpha}$ and define \textcolor{black}{$f(\mathbf U,\mathbf V) = \sum_{ij}\big[E g(\mathbf Z_{ij};\boldsymbol{\alpha})\mathbf u_i^\top \mathbf v_j - \mathbf Y_{ij}\log(\mathbf u_i^\top \mathbf v_j)\big]$}.
The subproblem in $(\mathbf U,\mathbf V)$ becomes
\[
\min_{\mathbf U,\mathbf V \geq \mathbf{0}} f(\mathbf U, \mathbf V) + \sum_{X \in \{UU, UV, VV\}} \lambda_X \lVert \mathbf M_X\rVert_{1/2},
\]
where $f$ is smooth but nonconvex in $(\mathbf U,\mathbf V)$ on the domain where $\mathbf u_i^\top \mathbf v_j > 0$ and $\lVert \mathbf A_X \rVert_{1/2}$'s are nonconvex and nonsmooth. To handle these terms, we introduce auxiliary variables
$(\mathbf A_{UU}, \mathbf A_{VV}, \mathbf A_{UV})$ and enforce the equalities through constraints
$\mathbf A_{UU} = \mathbf U\mathbf U^\top, 
\mathbf A_{VV} = \mathbf V\mathbf V^\top, 
\mathbf A_{UV} = \mathbf U\mathbf V^\top.$
 
The $\mathbf U\mathbf V$-subproblem becomes the constrained problem
\begin{align}
{\min_{\mathbf U,\mathbf V{\ge} \mathbf{0}, \mathbf A_X}}  f(\mathbf U,\mathbf V)
{+} \sum_X \lambda_{X}\|\mathbf A_{X}\|_{\frac{1}{2}}  \text{ s.t. } \mathbf A_{X} {=} \mathbf M_X.
\label{eq:uv-constrained}
\end{align}

Without loss of generality, for numerical stability and implementation convenience, we adopt the scaled dual variables 
\textcolor{black}{$\mathbf W_X$} \cite{boyd2011distributed, wright2022high} and penalty $\textcolor{black}{\rho_X > 0}$ for $X \in \{UU, VV, UV\}$. The resulting scaled augmented Lagrangian is
\begin{multline*}
    \mathcal{L}_\rho (\mathbf U, \mathbf V, \mathbf A, \mathbf W) =  f(\mathbf U,\mathbf V) + \sum_X \big(\lambda_{X}\|\mathbf A_{X}\|_{1/2}  \\+\frac{\rho_{X}}{2} \left[\lVert \mathbf M_{X} - \mathbf A_{X} + \mathbf W_X \rVert_2^2 - \lVert \mathbf W_{X}\rVert_2^2  \right] \big)
\end{multline*}
 At iterate $k$, denote $F(\mathbf U^k, \mathbf V^k) {:=} f(\mathbf U^k, \mathbf V^k) + \sum_{X}\tfrac{\rho_{X}^k}{2}\|\mathbf M_X^k-\mathbf B_{X}^k\|^2$ with $\mathbf B_X^k :=\mathbf A_X^k-\mathbf W_X^k$
  \begin{enumerate}[label=(\alph*), leftmargin=2.0em,left=-3pt,itemsep=1pt, parsep=-1pt, topsep=-0pt, partopsep=-1pt]
  \item \textbf{$\mathbf U$-update:} Update $\mathbf U$ by minimizing the smooth objective plus the quadratic penalties:
   \vspace{-0.5em}
    \begin{multline*}
\hspace{-2ex}\mathbf U^{k+1}=\arg\min_{\mathbf U\ge \mathbf 0}
    f(\mathbf U,\mathbf V^k)+ \sum_{X}
    \tfrac{\rho_{X}^k}{2}\|\mathbf M_X-\mathbf B_{X}^k\|^2,
    \end{multline*}
    \vspace{-0.5em}
    
    where the partial gradient is 
    $\nabla_{\mathbf U} F(\mathbf U, \mathbf V^k)  =   \textcolor{black}{E (\mathbf{\tilde P}^{k+1})}\mathbf V^k-(\mathbf Y/\mathbf U\mathbf V^{k\top})\mathbf V^k  + 2\rho_{UU}^k(\mathbf U\mathbf U^\top-\mathbf B_{UU}^k)\mathbf U
    + \rho_{UV}^k(\mathbf U\mathbf V^{k\top}-\mathbf B_{UV}^k)\mathbf V^k$. Using a projected gradient step where the step size is determined by backtracking line search:
    \[
    \mathbf U^{k}_{s+1}=\max(\mathbf U^{k}_s-t_U^{s,k}\nabla_{\mathbf U} F(\mathbf U^k_{s}, \mathbf V^k),0),
    \]
    where at time $s = 0$, $\mathbf U^k_0 =\mathbf U^k$ and $\mathbf U^k_{\infty} = \mathbf U^{k+1}$.
  \item \textbf{$\mathbf V$-update:} Similarly, $\mathbf V$-subproblem is
\begin{equation*}
\hspace{-2ex}\mathbf V^{k+1}=
\arg\min_{\mathbf V\ge \mathbf 0}
f(\mathbf U^{k+1},\mathbf V)
+\sum_{X}
    \tfrac{\rho_{X}^k}{2}\|\mathbf M_X-\mathbf B_{X}^k\|^2,
\end{equation*}
with partial gradient
$\nabla_{\mathbf V} F(\mathbf U^{k+1}, \mathbf V) = \textcolor{black}{E (\tilde{\mathbf P}^{k+1})^{\top}} \mathbf  U^{k+1}-(\mathbf Y/\mathbf U^{k+1}\mathbf V^\top)^\top \mathbf U^{k+1}
+2\rho_{VV}^k(\mathbf V\mathbf V^\top-\mathbf B_{VV}^k)\mathbf V
+\rho_{UV}^k(\mathbf V(\mathbf U^{k+1})^\top-\mathbf B_{UV}^{k\top})\mathbf U^{k+1}$. Solving with projected gradient descent with step size updated via backtracking line search,
 \vspace{-0.5em}
\[
\mathbf V^k_{s+1}=\max(\mathbf V^k_s-t_V^{s,k}\nabla_{\mathbf V} F(\mathbf U^{k+1}, \mathbf V^k_s),0),
\]
where at time $s = 0$, $\mathbf V^k_0 = \mathbf V^k$ and $\mathbf V^k_{\infty} = \mathbf V^{k+1}$.

\item \textbf{$\mathbf A$-subproblems:} The $\mathbf A_X$-update solves
\begin{multline*}
\hspace{-2ex}\mathbf A_X^{k+1} {=} \arg\min_{\mathbf A_X}
\lambda_X \|\mathbf A_X\|_{\frac 1 2}{+}
\frac{\rho_X^k}{2}
\|\mathbf M_X^{k+1} {-} \mathbf A_X {+}\mathbf W_X^k)\|_F^2 .
\end{multline*}
Since the objective is separable across entries, each element solves
$a^{k+1} = \arg\min_a \lambda_X \|a\|_{1/2} +
({\rho_X^k}/{2})(a-b)^2 $ where
$b=(\mathbf M_X^{k+1}+\mathbf W_X^k)_{ij}$. The solution is given by the half-thresholding operator~\cite{xu2012lhalf} with 
\begin{equation*}
a^{k+1} =
\begin{cases}
0, & |b| \le \tau_X^k, \\
\mathcal{H}_{\lambda_X/\rho_X^k}(b), & |b| > \tau_X^k,
\end{cases}
\end{equation*}
where $\tau_X^k = ({3}/{2})
\left({\lambda_X}/{\rho_X^k}\right)^{2/3}$ and $\mathcal{H}_{\lambda/\rho}(\cdot)$ denotes the half-thresholding map. $\tau_X^k$ is determined by solving the entrywise $\mathbf A$-subproblem so that the half-thresholding method is an exact global minimizer for the $\mathbf A$-subproblem, which will be shown in Lemma~\ref{lem:exact_a}.

  \item \textbf{Dual updates:}  \textcolor{black}{Denote the unscaled dual variables as $\mathbf H_X$.} Since the penalty parameter $\rho_X^k$ varies across iterations, for the unscaled dual update $\mathbf H_X^{k+1} = \rho_X^k \left(\mathbf W_{X}^k-\mathbf A_{X}^{k+1}+\mathbf M_X^{k+1}\right)$, the scaled dual variables must be adjusted  to align with $\rho_X^{k+1}$, leading to
 \begin{equation*}
  \mathbf W_{X}^{k+1}=({\rho_X^{k}}/{\rho_X^{k+1}})\left(\mathbf W_{X}^k-\mathbf A_{X}^{k+1}+\mathbf M_X^{k+1}\right).
\end{equation*}
\end{enumerate}
\subsection{KKT Conditions of the Split Formulation}
Following the general ADMM framework~\cite{wright2022high}, we define entrywise indicator $\iota_{\geq 0}(x) = 0$ if $x\geq 0$ and $+\infty$ otherwise, and similarly for $\iota_{[0,1]^{IJ}}$.
For $X \in \{UU, UV, VV\}$ the first-order optimality condition for~\ref{eq:overall_problem} can be written as:
\begin{subnumcases}{}
    \mathbf{0} \in \nabla_{\mathbf U} F(\mathbf U,\mathbf V) + \partial \iota_{\geq 0}(\mathbf U), \label{eq:U_station} \\
    \mathbf{0} \in \nabla_{\mathbf V} F(\mathbf U,\mathbf V) + \partial \iota_{\geq 0}(\mathbf V), \label{eq:V_station} \\
    \mathbf{0} \in \partial \psi(\mathbf A_X) - \rho_X\mathbf W_X, \label{eq:Y_station} \\
    \mathbf 0 = \mathbf M_X - \mathbf A_X. \label{eq:primalease}\\
    \textcolor{black}{\mathbf{0} \in \mathbf{Z}^\top \mathrm{vec}\big(\nabla\phi(\tilde{\mathbf P};\boldsymbol\Lambda)
      + \partial\iota_{[0,1]^{I\times J}}(\tilde{\mathbf P})\big), \label{eq:kkt_alpha}}\\
    \textcolor{black}{\mathbf{0} = \mathrm{vec}(\tilde{\mathbf P}) - \mathbf Z\boldsymbol\alpha, \label{eq:kkt_P}}
\end{subnumcases}
where $\psi(\mathbf A_X)=\lambda_X\|\mathbf A_X\|_{1/2}$ \textcolor{black}{and $\boldsymbol\Lambda=\mathbf U\mathbf V^\top$.
Conditions \eqref{eq:kkt_alpha}--\eqref{eq:kkt_P} are the KKT conditions of the detection subproblem~\eqref{eq:alpha_subprob} in $(\boldsymbol\alpha,\tilde{\mathbf P})$.
For a nonlinear $g$ with parameter $\boldsymbol{\theta}$, \eqref{eq:kkt_alpha}--\eqref{eq:kkt_P} are replaced by $\mathbf 0\in\nabla_{\boldsymbol\theta}\big[\phi(g(\mathbf Z;\boldsymbol\theta);\boldsymbol\Lambda)+r(\boldsymbol\theta)\big]+\mathcal N_\Theta(\boldsymbol\theta)$.}
\vspace{-2ex}

\subsection{Increasing Penalization} \label{par:inc_pen} To guarantee asymptotic feasibility, we adopt an adaptive penalization strategy~\cite{yuanadmm, shi2020penalty, hallak2023adaptive}. Given a positive vanishing residual tolerance $\{\epsilon_X^k\}$ with $\sum_{k=1}^\infty (\epsilon_X^k)^2 < \infty$ and an update factor $\gamma>1$, we set
\begin{align*}
    \rho_X^{k+1}=
    \begin{cases}
\gamma \rho_X^k, & \text{if } \| r_X^{k+1}\|_F > \epsilon_X^k,\\
\rho_X^k, & \text{otherwise}.
\end{cases}
\end{align*}
where $r_X^{k+1} := \mathbf M_X^{k+1} - \mathbf A_X^{k+1}$ denotes the primal feasibility residual. The penalty is increased only when the current feasibility violation exceeds the prescribed tolerance; otherwise, it remains unchanged. 
To ensure the penalty updates occur only finitely many times, we additionally require that $\{\epsilon_X^k\}$ decays slower than the residual bound, which will be mentioned in the convergence analysis in Lemma~\ref{lem:bound_y}.
\begin{remark}\label{rem:impute}
For datasets with structurally missing entries (rather than observed zeros), we use an EM-style approach to handle them. \textcolor{black}{Let $x_{ij}=\mathbf u_i^\top\mathbf v_j$ for the variable of the $({\mathbf{U}},{\mathbf{V}})$-subproblem, and $x^k_{ij}=\mathbf u_i^{k\top}\mathbf v^k_j$ and $\mathbf P^k_{ij}=g(\mathbf Z_{ij};\boldsymbol\theta^k)$ for their values at iteration $k$. Each missing entry of each replicate is imputed by its conditional expectation $\mathbf P^k_{ij}\max(x^k_{ij},\delta)$, so $\widehat{\mathbf Y}^k_{ij}=E\,\mathbf P^k_{ij}\max(x^k_{ij},\delta)$, and a missing cell contributes
$q^k_{ij}(x_{ij})=E\mathbf P^k_{ij}\max(x_{ij},\delta)-\widehat{\mathbf Y}^k_{ij}\log\max(x_{ij},\delta)$
to the subproblem. Since $q^k_{ij}(x_{ij})\ge q^k_{ij}(x^k_{ij})$ for all $x_{ij}$, every update satisfies $f_\Omega({\mathbf{U}},{\mathbf{V}})-f_\Omega({\mathbf{U}}^k,{\mathbf{V}}^k)\le F_k({\mathbf{U}},{\mathbf{V}})-F_k({\mathbf{U}}^k,{\mathbf{V}}^k)$, where $f_\Omega$ is the observed-data objective.}
\end{remark}

\section{Convergence Analysis}\label{sec:converg:analysis}
\paragraph*{Overview} \textcolor{black}{The analysis consists of five steps. (i) The $\mathbf{A}_X$-update is an exact global minimizer of its subproblem (Lemma \ref{lem:exact_a}). (ii) The penalty parameters increase only finitely many times, and the dual variables stay bounded (Lemma \ref{lem:bound_y}). (iii) The augmented Lagrangian is bounded below along the iterates (Lemma~\ref{lem:bound_L}) and decreases up to summable dual terms (Lemma~\ref{lem:alm_desc}). (iv) Iterative differences of $\mathbf U$ and $\mathbf V$ and the feasibility residuals vanish (Lemma \ref{lem:vanishing-increments}). (v) Passing to the limit along a convergent subsequence gives KKT conditions at every cluster point (Theorem \ref{prop:kkt-cluster}). Proofs are deferred to Appendix \ref{app:proof_conv}.}

\begin{assumption}
\label{assum:bounded-iterates}
We assume that $\{({\mathbf{U}}^k,{\mathbf{V}}^k, {\mathbf{A}}^k)\}$ is bounded where ${\mathbf{A}}^k := ({\mathbf{A}}_{UU}^k, {\mathbf{A}}_{UV}^k, {\mathbf{A}}_{VV}^k)$.
\end{assumption}

    Assumption~\ref{assum:bounded-iterates} is a typical presumption in the analysis of nonconvex optimization methods~\cite{el2025convergence, huang2016flexible, hallak2023adaptive} where primal variables are assumed to be bounded.

\begin{assumption}
\label{assum:err_con} 
The $\mathbf U$ and $\mathbf V$-subproblems are solved inexactly with vanishing first-order residuals, i.e., 
\[\mathrm{dist}{(}\mathbf 0,\nabla_{\mathbf U} F_k(\mathbf U^{k+1},\mathbf V^k)+\mathcal N_{\mathbb R_{+}}(\mathbf U^{k+1}){)}\le \varepsilon_k,\]
where $\varepsilon_k \to 0$ as $k \to \infty$, $\mathcal N_{\mathbb R_+}(\mathbf U^{k+1})$ is the normal cone of $\mathbf U^{k+1}$ at $\mathbb R_+$ and similarly for $\mathbf V$. 
\end{assumption}
Assumption~\ref{assum:err_con} means that there exists an error term $e_U^k$ with $\|e_U^k\|_F\le \varepsilon_k$ such that $\mathbf 0 \in \nabla_{\mathbf U} F_k(\mathbf U^{k+1},\mathbf V^k)+e_U^k+\mathcal N_{\mathbb R_+}(\mathbf U^{k+1})$, and analogously for the $\mathbf V$-block. Inexact updates with vanishing tolerances are a common technique in ADMM-based methods for handling nonconvex and nonsmooth optimization problems~\cite{fernandez2012local, el2025convergence, yang2019inexact}. 
The inner projected gradient updates for the $\mathbf U-$block are implemented
with standard backtracking line search initialized from a fixed maximal step size $t_U^{\max} > 0$ and terminated when the successive iterate change is sufficiently small or when the number of inner iterations reaches a preset cap $S_{\max}$, i.e., $\mathbf U^{k+1} = \mathbf U^k_{S_{\max}}$, and similarly for $\mathbf V$. 

\begin{assumption}
\label{assum:block_lipschitz}
Let $\delta > 0$. Define $\mathcal D_{\delta} := \{({\mathbf{U}},{\mathbf{V}}) \in \mathcal D: \mathbf u_i^\top \mathbf v_{j} \geq \delta, \forall (i,j)\textcolor{black}{\in\Omega} \text{ with } \mathbf{Y}_{ij} > 0\}$
where $\mathcal{D} \subset \mathbb{R}^{I \times F} \times \mathbb{R}^{J \times F}$ is the bounded iterate set for $\{({\mathbf{U}}^k, {\mathbf{V}}^k)\}$.
Then:
\begin{itemize}[leftmargin=1.0em,left=0pt,itemsep=0.5pt, parsep=-1pt, topsep=-0pt, partopsep=-1pt]
    \item For all $({\mathbf{U}}, {\mathbf{V}}) \in \mathcal D_\delta$, the partial gradient $\nabla_{\mathbf{U}} f({\mathbf{U}}, {\mathbf{V}})$, $\nabla_{\mathbf{V}} f({\mathbf{U}},{\mathbf{V}})$ exists and are block-Lipschitz continuous. That is, there exists some constants $L_U$ and $L_V$ such that $\forall ({\mathbf{U}}_1, {\mathbf{V}})$, $({\mathbf{U}}_2, {\mathbf{V}}) \in \mathcal D_\delta$, $\lVert \nabla f_U({\mathbf{U}}_1, {\mathbf{V}}) - \nabla f_U({\mathbf{U}}_2, {\mathbf{V}})\rVert_F\leq L_U \lVert {\mathbf{U}}_1 - {\mathbf{U}}_2\rVert_F$ and similarly for ${\mathbf{V}}$.
    \item Across iterates, the smooth part of the augmented Lagrangian has block-Lipschitz partial gradient on $\mathcal D_\delta$ with constants $L_U^k, L_V^k {<} {\infty}$ as $\mathcal D_\delta$ is bounded by Assumption~\ref{assum:bounded-iterates}. 
\end{itemize}
\end{assumption}

\textcolor{black}{Assumption~\ref{assum:block_lipschitz} restricts only observed cells with a positive count. It is consistent with the model: there are no false positives, so $\mathbf{Y}_{ij}>0$ implies $\lambda_{ij}>0$, and since $E\,\mathbf P_{ij}x-\mathbf{Y}_{ij}\log x\to+\infty$ as $x\to0$, descent iterates are repelled from such points. 
In practice, the rate on every positive-count cell, observed or imputed, is evaluated as $\max(x_{ij},\delta)$, and the $({\mathbf{U}},{\mathbf{V}})$-gradient is the exact gradient of this floored objective.
On observed cells, the floor is only a numerical safeguard.  Under Assumption~\ref{assum:block_lipschitz}, it is inactive along the iterates.
On missing cells, the sparsity penalty can drive $x_{ij}$ to zero where the unfloored term $\widehat{\mathbf Y}^k_{ij}\log x_{ij}$ has unbounded curvature. With the floor, $q^k_{ij}$ is constant for $x_{ij}\le\delta$ with curvature at most $E\mathbf P^k_{ij}\max(\bar\Delta,\delta)/\delta^2$ for $x_{ij}>\delta$, where $\bar\Delta$ bounds $x_{ij}$ along the iterates (Assumption~\ref{assum:bounded-iterates}). At $x_{ij}=\delta$, its slope can only drop. The quadratic upper bound of the descent lemma therefore holds with $L_U^k$ and $L_V^k$ bounded uniformly in $k$. Missing cells do not enter the likelihood, so the floor does not change the model fitted to the observed data.
In both cases, the floor acts only on the likelihood term, not on ${\mathbf{U}}{\mathbf{V}}^\top$ or on ${\mathbf{A}}_X$ since the ${\mathbf{A}}_X$-update (Lemma~\ref{lem:exact_a}) does not involve it, and the $\ell_{1/2}$ penalty can still set entries of ${\mathbf{A}}_X$ to zero.}
The inner projected gradient iterates in the ${\mathbf{U}}/{\mathbf{V}}$-updates remain in $\mathcal D_\delta$, so the block Lipschitz constants $L_U^k$ and $L_V^k$ apply to every gradient evaluated in the inner loops.

\begin{assumption}\label{assum:alpha_inexact}
\textcolor{black}{Let $\Theta$ be closed, and let $\varphi_k(\boldsymbol\theta) := \phi(g({\mathbf{Z}};\boldsymbol\theta);\lambda^k)+r(\boldsymbol\theta)$ be the detection objective at outer iteration $k$, with $\lambda^k = {\mathbf{U}}^k{\mathbf{V}}^{k\top}$. The detection update returns $\boldsymbol\theta^{k+1}\in\Theta$, and:
\begin{enumerate}[label=(\alph*), leftmargin=2.0em,left=0pt,itemsep=1pt, parsep=-1pt, topsep=-0pt, partopsep=-1pt]
\item $g(z;\cdot)$ and $r$ are continuously differentiable on $\Theta$, and $\{\boldsymbol\theta^k\}$ is bounded;
\item  $g({\mathbf{Z}}_{ij};\boldsymbol\theta)>0$ for all $\boldsymbol\theta\in\Theta$ and all $(i,j)$ with $\mathbf{Y}_{ij}>0$;
\item the update is approximately stationary: $\mathrm{dist}\big(\mathbf 0{,}\nabla\varphi_k(\boldsymbol\theta^{k+1}){+}\mathcal N_\Theta(\boldsymbol\theta^{k+1})\big)\le\xi_k$, with $\xi_k\to0$.
\end{enumerate}}
\end{assumption}

    In~\cite{fu2019link}, the convex $\alpha$-subproblem is solved exactly; here we allow inexact updates \textcolor{black}{and nonlinear detection models. \textit{(b)} is a property of the model and can be checked in advance. Since every cluster point lies in $\Theta$ by \textit{(a)}, \textit{(b)} keeps the fitted detection probability positive wherever a detection was recorded. It holds for a network with a sigmoid output, as in~\cite{ibrahim2023under}, and for the linear model whenever the box constraint on ${\mathbf{Z}}\boldsymbol\alpha$ has a strictly positive lower end. This only excludes detection probabilities that are numerically zero, which the likelihood already rules out on positive-count cells. In the linear case, $\Theta$ is a polyhedron, $r\equiv0$, and $\varphi_k$ is convex, so any stationary point of the detection subproblem is a global minimizer. For the network, \textit{(a)} holds with smooth activations and weight decay $r(\boldsymbol\theta)=(w/2)\|\boldsymbol\theta\|_2^2$, and \textit{(c)} is enforced by running the inner solver until its gradient norm is at most $\xi_k$.}

\begin{lemma} \label{lem:exact_a}
For each $X\in\{UU,UV,VV\}$ and iteration $k$, the $\mathbf A$-update given by the entrywise half-thresholding operator with threshold
$\tau_X^k=({3}/{2})({\lambda_X}/{\rho_X^k})^{2/3}$
is a global minimizer of the $\mathbf A$-subproblem.
\end{lemma}

\begin{lemma}
\label{lem:bound_y}
Fix $X\in\{UU,UV,VV\}$ and let $n_X$ be the number of entries of $\mathbf A_X$. Define $C_X:=\frac{3}{2}\sqrt{n_X}\,\lambda_X^{2/3}$.  Assume $\sum_{k=1}^\infty \epsilon_X^k<\infty$ and the threshold sequence decays slower than $(\rho_X^k)^{-2/3}$ in the sense that whenever $\rho_X^k\to\infty$ along a subsequence, one has $\epsilon_X^k \ge 2C_X(\rho_X^k)^{-2/3}$ for all sufficiently large $k$ along that subsequence. Then there exists $K_X\in\mathbb N$ and $\bar\rho_X>0$ such that $\rho_X^k=\bar\rho_X$, $\|r_X^{k+1}\|_F\le \epsilon_X^k$, $\forall k\ge K_X$. \textcolor{black}{Moreover, the scaled dual sequence $\{\mathbf W_X^k\}$ is bounded.}
\end{lemma}

\textcolor{black}{The decay condition couples the tolerance with the penalty: the tolerance must vanish more slowly than the residual bound implied by the penalty so that the penalty and dual variables stay bounded. Similar conditions are used in augmented Lagrangian methods, where the feasibility tolerance is linked to powers of the penalty parameter to bound penalty parameters~\cite{conn1991globally, hallak2023adaptive}.}

\begin{lemma}
\label{lem:bound_L}
Along the iterates, the augmented Lagrangian $\mathcal L_{\rho^k}(\mathbf U^k,\mathbf V^k,\mathbf A^k,\mathbf W^k)$ are bounded from below.
\end{lemma}

\begin{lemma} \label{lem:alm_desc}
\textcolor{black}{For all $K\ge 1$ and fixed $\{\boldsymbol{\alpha}^k\}_{k = 0}^K$, there exist constants $c_U, c_V>0$ such that the iterates generated by the $(\mathbf U,\mathbf V)$-gradient step, the $\mathbf A$-update, and the  $\mathbf W$-update satisfy}
\begin{multline*}\label{eq:ergodic-almost-descent}
\mathcal L_{\rho^K}(\mathbf U^K,\mathbf V^K,\mathbf A^K,\textcolor{black}{\mathbf W^K})-\mathcal L_{\rho^0}(\mathbf U^0,\mathbf V^0,\mathbf A^0,\textcolor{black}{\mathbf W^0}) \\
 \le
-\sum_{k=0}^{K-1}\Big(c_U\|\Delta \mathbf U^{k}\|_F^2
+c_V\|\Delta \mathbf V^{k}\|_F^2\Big) \\
+\textcolor{black}{\sum_{k=0}^{K-1}\sum_{X}
\frac{\rho_X^{k+1}+\rho_X^k}{2}
\|r_X^{k+1}\|_F^2},
\end{multline*}
where $\|\Delta \mathbf U^{k}\| {=} \|\mathbf U^{k+1} - \mathbf U^k\|$ and similarly for $\Delta \mathbf V^k$.
\end{lemma}

\begin{lemma}\label{lem:vanishing-increments}
Under Assumptions 1--3 and the threshold condition of Lemma~\ref{lem:bound_y}, $\|\Delta \mathbf U^k\|_F\to 0$, $\|\Delta \mathbf V^k\|_F\to 0$, $\|r_X^k\|_F\to 0$ for all $X\in\{UU,UV,VV\}$.
\end{lemma}

\begin{theorem}
\label{prop:kkt-cluster}
Let Assumptions~\ref{assum:bounded-iterates}--\ref{assum:alpha_inexact}, and the threshold condition of Lemma~\ref{lem:bound_y} hold.
Define $\mathcal S^k := (\mathbf U^k,\mathbf V^k,\mathbf A^k,\textcolor{black}{\mathbf W^k,\boldsymbol\alpha^k})$, $\mathbf A^k := (\mathbf A_{UU}^k,\mathbf A_{UV}^k,\mathbf A_{VV}^k)$, \textcolor{black}{$\mathbf W^k := (\mathbf W_{UU}^k,\mathbf W_{UV}^k,\mathbf W_{VV}^k)$}. Then $\{\mathcal S^k\}$ is bounded and has a cluster point. \textcolor{black}{Any cluster point $(\mathbf U^\star,\mathbf V^\star,\mathbf A^\star,\mathbf W^\star,\boldsymbol\alpha^\star)$ of $\{\mathcal S^k\}$, together with $\mathrm{vec}(\tilde{\mathbf P}^\star)=\mathbf Z\boldsymbol\alpha^\star$ and $\rho_X=\bar\rho_X$, satisfies the KKT system \eqref{eq:U_station}--\eqref{eq:kkt_P}. For a nonlinear $g$, the same holds with $\boldsymbol\alpha^k$ replaced by $\boldsymbol\theta^k$ and \eqref{eq:kkt_alpha}--\eqref{eq:kkt_P} replaced by their $\boldsymbol\theta$-counterpart.}
\end{theorem}
\paragraph*{Relation to existing analyses} \textcolor{black}{We go beyond the existing nonconvex ADMM literature in two ways:
1) Splitting $\lambda_X\|\mathbf M_X(\mathbf U,\mathbf V)\|_{1/2}$ yields a non-Lipschitz block coupled to the factors through $\mathbf A_X=\mathbf M_X$. Analyses with affine~\cite{hong2016convergence, wang2019global, yuanadmm} or nonlinear~\cite{el2025convergence} constraints bound the dual increment by the increment of the block that enters the constraint linearly, which requires that block to have a smooth objective. Here, the objective of $\mathbf A_X$ is the nonsmooth, non-differentiable $\ell_{1/2}$ norm, which also violates the continuously differentiable outer function requirement for the penalty dual decomposition method~\cite{shi2020penalty}. We thus control the dual by the penalty with the fact that half-thresholding moves each entry by at most $\tfrac32(\lambda_X/\rho_X)^{2/3}$ (Lemma~\ref{lem:exact_a}), forcing the penalty to stabilize after finitely many increases (Lemma~\ref{lem:bound_y}), after which the standard descent argument applies;
2) Each dual update follows a single pass over the primal blocks, and every block except $\mathbf{A}_X$-block is solved inexactly (Assumptions~\ref{assum:err_con} and~\ref{assum:alpha_inexact}). However, \cite{fu2019link} assumes exact minimization of the detection subproblem.}
\section{Numerical Analysis on Synthetic and Real-World Ecological Datasets}\label{sec:experiments}
\subsection{Synthetic Experiment} 
\paragraph{Setup}
\textcolor{black}{Synthetic datasets ($I=J=30, F=8$) are generated from the model of Section \ref{sec:model} where entries of true latent factor matrices $\mathbf U_0$ and $\mathbf V_0$ are drawn from a uniform distribution $\mathcal{U}(0, \bar u)$. Sparsity is strictly enforced at $s \in (0,1)$ by zeroing entries where a random draw $r \sim \mathcal{U}(0,1) < s$. 
We use two sparsity level: sparser factors with $s = 0.8$ and $\bar u = 15$ and denser factors with $s = 0.5$ and $\bar u = 6$.}
To maintain the network structure, any resulting all-zero rows are repopulated with a single non-zero value drawn from the same uniform distribution. The intensity matrix is $\boldsymbol\Lambda_{0} = \mathbf U_0 \mathbf V_0^{\top}$, and the observation probability matrix $\mathbf P_{0} = \mathbf Z \boldsymbol{\alpha}_0$, where feature vectors $Z$ and weights $\boldsymbol\alpha_0$ \textcolor{black}{with $R = 10$} features are sampled from $\mathcal{U}(0, 1)$ and are row-normalized. 
\textcolor{black}{Observed counts $\mathbf Y$ are drawn as $\mathbf N \sim \text{Poisson}(\boldsymbol\Lambda_{0})$ and $\mathbf Y \sim \text{Binomial}(\mathbf N, \mathbf P_{0})$  with one replicate ($E= 1$). A random subset of cells is hidden from all methods, with missing rates of $\{0, 5, 20\}\%$ to mimic truly missing entries. At $s = 0.8$, about $67\%$ of the entries are zeros, and at $s = 0.5$, these are $14\%$.} We use $\lambda_{UU} = \lambda_{VV} = \lambda_{UV} = 10^{-2}$, $\rho_{UU}^0 = \rho_{VV}^0 = \rho_{UV}^0 = 10^{-3}$, \textcolor{black}{$\gamma = 1.001$,} $p_0 = 0.5$, and initialization of Appendix \ref{sec:init}.
\paragraph{Baselines}
We compare our proposed method against five baselines, including (1)  Poisson NMF~\cite{lee2000algorithms}: a classical non-negative matrix factorization that assumes the observed counts follow a Poisson distribution without imperfect detection; \textcolor{black}{(3) Sparse NMF~\cite{hoyer2004non}: Poisson NMF with an $\ell_1$ penalty on $\mathbf U$ and $\mathbf V$ with no detection model; } (2) N-Mixture~\cite{fu2019link}: the collaborative filtering approach which pairs a latent Poisson model with a Binomial observation layer without sparsity structures; \textcolor{black}{(4) Robust PCA: principal component pursuit with missing entries~\cite{candes2011robust} returns a low-rank matrix $\mathbf{L}$, where $\max(\mathbf{L},0)$ is a least-squares estimate of the observed counts; and (5) UncleTC~\cite{ibrahim2023under}: under-counted completion with a neural detection model. 
N-mixture, Poisson NMF, and UncleTC are run with their published code, and Sparse NMF and Robust PCA with our implementations with SVD initialization \cite{boutsidis2008svd}. Every method runs for at most 100 iterations or until its objective is stationary, and no early stopping is used on held-out data (UncleTC uses its own convergence tolerance). To keep N-mixture's iterates finite when entries are missing, their refilled values are capped at the largest observed count. All factor models use the true rank $F=8$, and the penalty weights of Sparse NMF and Robust PCA are tuned on 10 additional trials. 
}
\paragraph{Metrics} To evaluate recovery of the latent factors, we use the permutation-invariant MSE ~\cite{fu2019link}.
For $\mathbf U$, the metric is
\begin{equation}\label{eq:perm_mse}
    \hspace{-0.2cm}\operatorname{MSE}(\mathbf U){=}
\min_{\pi} \frac{1}{F} \sum_{f=1}^{F} \left\|
\frac{\mathbf U_0(:,\pi(f))}{\|\mathbf U_0(:,\pi(f))\|_2}{-}
\frac{ \mathbf U(:,f)}{\| \mathbf U(:,f)\|_2} \right\|_2^2,
\end{equation}
where $\pi$ ranges over all permutations of $\{1,\dots,F\}$. The same definition applies to $\mathbf V$. \textcolor{black}{For detection, we use the \textcolor{black}{scale-aligned} $\operatorname{MSE}(\mathbf P)= \min_{c>0}\frac{1}{IJ}\|c\mathbf Z\widehat{\boldsymbol\alpha}-\mathrm{vec}(\mathbf P_0)\|_2^2$. Normalized evaluations are} appropriate here because factor models are only identifiable up to column permutation and scaling~\cite{fu2019link}, \textcolor{black}{also discussed in Section \ref{sec:ident}.
For the estimated observed counts, we define the relative error as
\begin{equation*}\label{eq:rel_mean}
\mathrm{err}\big(\widehat{\mathbf{Y}}\big)=\big\|\widehat{\mathbf{Y}}-\mathbf{P_0}\odot{\mathbf{U_0}}{\mathbf{V_0}}^\top\big\|_F/\big\|\mathbf{P_0}\odot{\mathbf{U_0}}{\mathbf{V_0}}^\top\big\|_F.
\end{equation*}
The error is computed over all $I\times J$ cells, including the missing ones. All methods are scored at their final iterate. The reported results are means over 50 trials.
} 

\begin{figure}[ht]
\vspace{-0.225cm}
    \centering
    \includegraphics[width=0.9\linewidth]{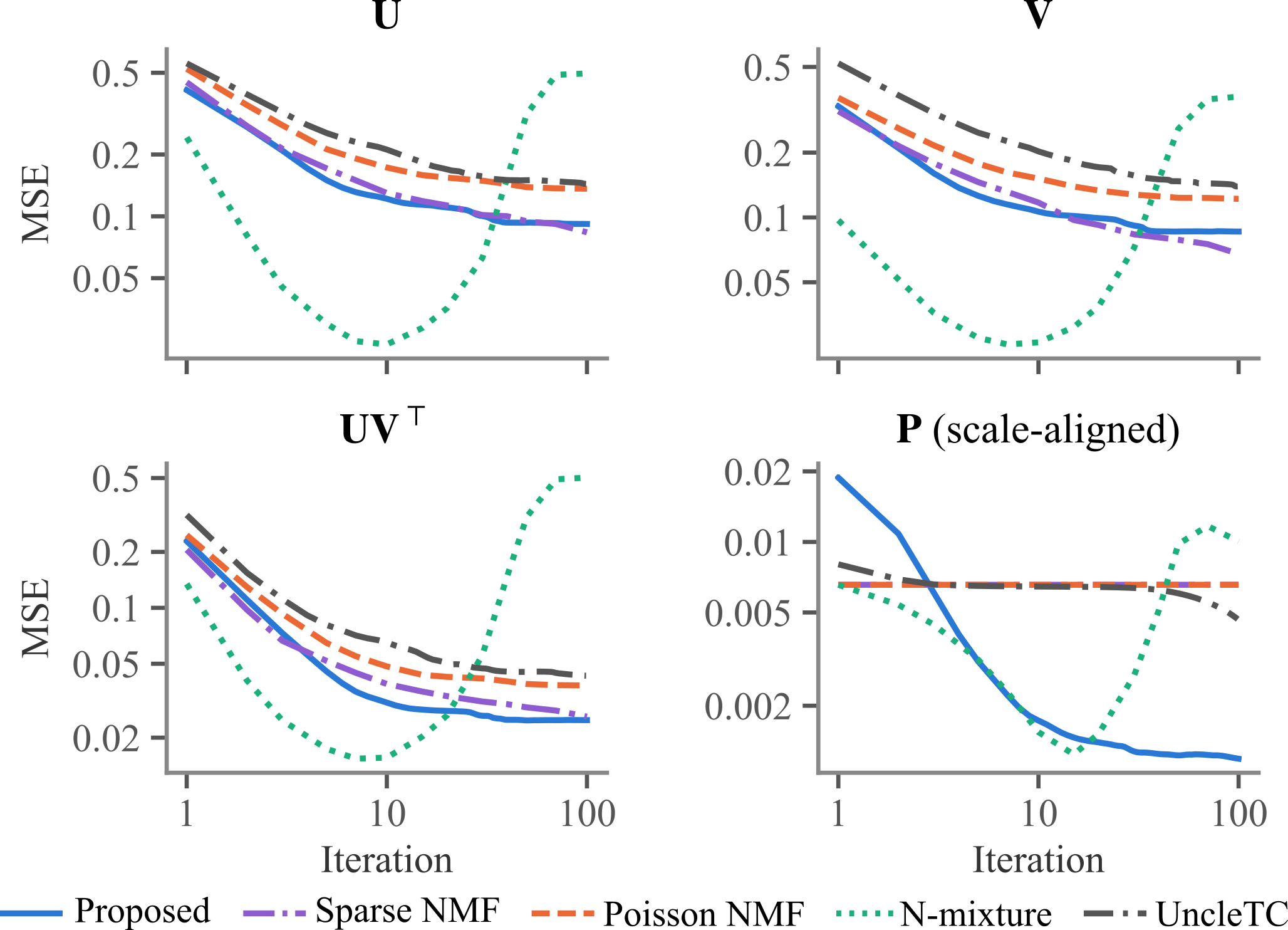}
    \vspace{-0.075em}
    \caption{\textcolor{black}{Error of ${\mathbf{U}}$, ${\mathbf{V}}$, ${\mathbf{U}}{\mathbf{V}}^\top$ and scale-aligned ${\mathbf{P}}$ versus iteration, 80\% factor sparsity, 5\% missing entries. Sparse/Poisson NMF omit detection probabilities, so their ${\mathbf{P}}$ curves overlap.}}
    \label{fig:best_mse}
    \vspace{-0.4em}
\end{figure}

\begin{table*}[t]
\centering
\scriptsize
\caption{\textcolor{black}{Recovery on synthetic data: MSE of ${\mathbf{U}}{\mathbf{U}}^\top$, ${\mathbf{V}}{\mathbf{V}}^\top$, ${\mathbf{U}}{\mathbf{V}}^\top$, relative error of $\mathbf{\widehat{Y}}$ ($\mathbf{\widehat{Y}}$), and scale-aligned detection error (${\mathbf{P}}$, $\times10^{-3}$). Robust PCA is scored on $\mathbf{\widehat{Y}}$. Methods without a detection model has constant ${\mathbf{P} = \mathbf{1}}$, with error 6.6 (80\% sparsity) and 6.2 (50\% sparsity). }}
\label{tab:mse_sim_con}
\setlength{\tabcolsep}{2.5pt}
\resizebox{\textwidth}{!}{%
\begin{tabular}{ll ccccc cccc c cccc ccccc ccccc}
\toprule
& & \multicolumn{5}{c}{Proposed} & \multicolumn{4}{c}{Sparse NMF} & \multicolumn{1}{c}{Robust PCA} & \multicolumn{4}{c}{Poisson NMF} & \multicolumn{5}{c}{N-mixture} & \multicolumn{5}{c}{UncleTC}\\
\cmidrule(lr){3-7}\cmidrule(lr){8-11}\cmidrule(lr){12-12}\cmidrule(lr){13-16}\cmidrule(lr){17-21}\cmidrule(lr){22-26}
sparsity & missing & $\mathbf U\mathbf U^\top$ & $\mathbf V\mathbf V^\top$ & $\mathbf U\mathbf V^\top$ & $\mathbf{\widehat{Y}}$ & $\mathbf P$ & $\mathbf U\mathbf U^\top$ & $\mathbf V\mathbf V^\top$ & $\mathbf U\mathbf V^\top$ & $\mathbf{\widehat{Y}}$ & $\mathbf{\widehat{Y}}$ & $\mathbf U\mathbf U^\top$ & $\mathbf V\mathbf V^\top$ & $\mathbf U\mathbf V^\top$ & $\mathbf{\widehat{Y}}$ & $\mathbf U\mathbf U^\top$ & $\mathbf V\mathbf V^\top$ & $\mathbf U\mathbf V^\top$ & $\mathbf{\widehat{Y}}$ & $\mathbf P$ & $\mathbf U\mathbf U^\top$ & $\mathbf V\mathbf V^\top$ & $\mathbf U\mathbf V^\top$ & $\mathbf{\widehat{Y}}$ & $\mathbf P$\\
\midrule
\multirow{3}{*}{\shortstack{80\%\\sparsity}}
& 0\%  & 0.048 & 0.042 & 0.018 & 0.093 & 1.1 & 0.057 & 0.060 & 0.024 & 0.129 & 0.105 & 0.109 & 0.132 & 0.030 & 0.136 & 0.065 & 0.078 & 0.013 & 0.082 & 1.2  & 0.126 & 0.129 & 0.041 & 0.125 & 4.3\\
& 5\%  & 0.060 & 0.054 & 0.025 & 0.105 & 1.2 & 0.067 & 0.062 & 0.026 & 0.138 & 0.124 & 0.119 & 0.139 & 0.038 & 0.159 & 0.521 & 0.407 & 0.502 & 1.73 & 10.1 & 0.122 & 0.140 & 0.043 & 33   & 4.7\\
& 20\% & 0.107 & 0.090 & 0.057 & 0.182 & 2.5 & 0.116 & 0.098 & 0.060 & 0.257 & 0.222 & 0.185 & 0.159 & 0.091 & 0.348 & 1.05 & 0.850 & 1.06 & 7.54 & 6.6  & 0.207 & 0.270 & 0.112 & 41   & 5.4\\
\midrule
\multirow{3}{*}{\shortstack{50\%\\sparsity}}
& 0\%  & 0.047 & 0.047 & 0.018 & 0.130 & 0.7 & 0.051 & 0.051 & 0.023 & 0.157 & 0.170 & 0.094 & 0.101 & 0.023 & 0.157 & 0.056 & 0.065 & 0.016 & 0.128 & 0.7  & 0.092 & 0.125 & 0.023 & 0.147 & 3.5\\
& 5\%  & 0.050 & 0.050 & 0.019 & 0.136 & 0.7 & 0.054 & 0.054 & 0.025 & 0.163 & 0.178 & 0.081 & 0.094 & 0.026 & 0.172 & 0.418 & 0.383 & 0.386 & 1.12 & 4.3  & 0.100 & 0.134 & 0.025 & 0.154 & 3.7\\
& 20\% & 0.080 & 0.071 & 0.039 & 0.191 & 1.1 & 0.289 & 0.301 & 0.255 & 0.532 & 0.229 & 0.116 & 0.105 & 0.052 & 0.256 & 0.600 & 0.573 & 0.569 & 2.64 & 6.0  & 0.125 & 0.162 & 0.052 & 0.406 & 4.2\\
\bottomrule
\end{tabular}}
\vspace{-1em}
\end{table*}

\paragraph{Results}
\textcolor{black}{From Table~\ref{tab:mse_sim_con}, the proposed method is the most accurate on ${\mathbf{U}}{\mathbf{U}}^\top$, ${\mathbf{V}}{\mathbf{V}}^\top$, and ${\mathbf{P}}$. The proposed method achieves the lowest errors at 5\% and 20\% missing entries and ranks second to N-mixture with no missing entries. The baselines are more accurate in some entries but lack parts of the model: robust PCA returns no factors, and N-mixture is more accurate only without missing entries. The proposed method also does not rely on sparse factors: with 50\% factor sparsity, its graph errors are similar to those with 80\% sparsity and smaller with 20\% missing entries, whereas sparse NMF over-shrinks the dense factors once 20\% of the entries are missing. From Fig.~\ref{fig:sim_recover_comparison}, Poisson NMF, N-mixture, and UncleTC concentrate ${\mathbf{V}}{\mathbf{V}}^\top$ on a few pairs. 
UncleTC models the detection probabilities by a neural network~\cite{ibrahim2023under}, and with linear ${\mathbf{P}}_0$ its detection error is higher than ours in every setting. Comparisons of nonlinear detections are provided in Appendix~\ref{app:nonlinear}.}

\textcolor{black}{From Fig.~\ref{fig:best_mse}, N-mixture reaches its lowest error after 8--10 iterations and then deteriorates steadily. It imputes the missing entries with ${\mathbf{U}}{\mathbf{V}}^\top$ without accounting for the detection probabilities, thus inflating the imputed counts. The 0\% rows of Table~\ref{tab:mse_sim_con} confirm this: without missing entries, N-mixture converges and is competitive.}

\begin{figure*}[t]
    \centering
    \includegraphics[width=0.9\linewidth]{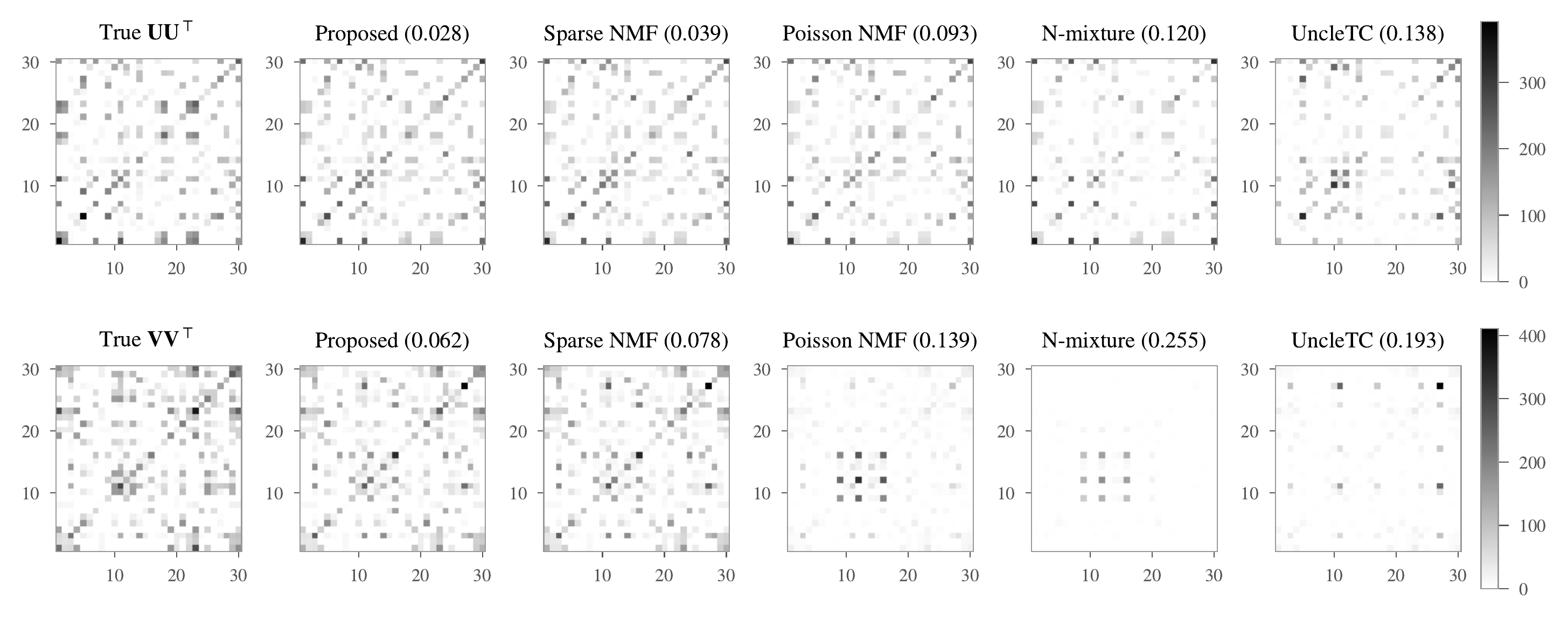}
    \vspace{-0.75em}
    \caption{\textcolor{black}{True and recovered ${\mathbf{U}}{\mathbf{U}}^\top$and ${\mathbf{V}}{\mathbf{V}}^\top$ for one synthetic trial (80\% factor sparsity, 5\% missing entries), chosen as the trial whose errors are closest to the medians over the 50 trials for all methods. Each estimate is rescaled to the truth using a least-squares scalar, and each row shares the same color scale. Titles give the normalized error~\eqref{eq:perm_mse} for this trial.}\vspace{-0.2cm}}
    \label{fig:sim_recover_comparison}
\end{figure*}

\subsection{Real-world Ecological Data Sets}
\subsubsection{Datasets}
We apply the proposed method to three distinct ecological datasets: Plant-Pollinator (PPI), Host-Parasite (HPI), and Location-Sonotype (Soundscape). Specific dataset details can be found in Appendix \ref{app:data_info}.

\subsubsection{Setup} 
\textcolor{black}{
All real-data results are held-out predictions. We use 30 random 80/10/10 training/validation/test splits, preserving the proportions of zero and positive counts. Hyperparameters are chosen by minimizing the rRMSE of the validation entries over the grids in Table~\ref{tab:grids}. For ours, $\rho_X^0 = 10^{-4}$, $p_0 = 0.5$, and $\lambda_{UU} = \lambda_{UV} = \lambda_{VV} = \lambda$ for HPI and Soundscape and $\lambda_{UU} = \lambda_{VV} = \frac{1}{5}\lambda_{UV} = \lambda$ for PPI. N-mixture and UncleTC follow the real-data protocols of~\cite{fu2019link} and~\cite{ibrahim2023under} and keep the iterate with the lowest validation rRMSE. 
Without this validation-based stopping and the cap used in the synthetic runs (Fig.~\ref{fig:best_mse}), its predictions overflow on every split. The other baselines run to convergence.}
\begin{table}[ht]
\vspace{-0.175cm}
\centering
\scriptsize
\caption{\textcolor{black}{Hyperparameter grids on real data.}}
\label{tab:grids}
\setlength{\tabcolsep}{2.5pt}
\begin{tabular}{@{}ll@{}}
\toprule
Method & Grid \\
\midrule
Proposed & $F\in\{5,10,15,20\}$, $\lambda\in\{0.003,0.01,0.03\}$ \\
Sparse NMF & $F\in\{5,10,15,20\}$, $\ell_1$ weight $\in\{0.01,0.1,1\}$ \\
Poisson NMF & $F\in\{5,10,15,20\}$ \\
Robust PCA & $a\in\{1,3,10,30,100,300\}\times10^{-3}$, $b\in\{0.5,1,2,4,\infty\}$ \\
N-mixture & $F\in\{2,5,10,20,40\}$ \\
UncleTC & $F\in\{5,10,15,20\}$, $\mu\in\{500,1000,2000,3000\}$\\
\bottomrule
\end{tabular}
\vspace{-0.175cm}
\end{table}

\subsubsection{Results}
\begin{table*}[htbp]
\centering
\caption{Held-out results on real data: mean $\pm$ sd over 30 random 80/10/10 splits.}
\label{tab:real}
\scriptsize
\setlength{\tabcolsep}{4pt}
\begin{tabular}{l ccc ccc ccc}
\toprule
& \multicolumn{3}{c}{HPI ($49\times19$)} & \multicolumn{3}{c}{PPI ($50\times50$)} & \multicolumn{3}{c}{Soundscape ($17\times290$)}\\
\cmidrule(lr){2-4}\cmidrule(lr){5-7}\cmidrule(lr){8-10}
Method & rRMSE & AUROC & AUPRC & rRMSE & AUROC & AUPRC & rRMSE & AUROC & AUPRC\\
\midrule
\textbf{Proposed} & 4.56 $\pm$ 1.19 & 0.75 $\pm$ 0.07 & 0.52 $\pm$ 0.10 & 5.52 $\pm$ 1.38 & 0.73 $\pm$ 0.03 & 0.70 $\pm$ 0.04 & 1.89 $\pm$ 0.67 & 0.93 $\pm$ 0.02 & 0.995 $\pm$ 0.002\\
Sparse NMF & 5.48 $\pm$ 2.97 & 0.70 $\pm$ 0.06 & 0.45 $\pm$ 0.10 & 7.23 $\pm$ 5.26 & 0.70 $\pm$ 0.04 & 0.68 $\pm$ 0.04 & 2.15 $\pm$ 0.98 & 0.94 $\pm$ 0.02 & 0.995 $\pm$ 0.002\\
Robust PCA & 4.48 $\pm$ 1.26 & 0.76 $\pm$ 0.07 & 0.56 $\pm$ 0.09 & 5.44 $\pm$ 1.37 & 0.71 $\pm$ 0.04 & 0.70 $\pm$ 0.05 & 1.49 $\pm$ 0.28 & 0.92 $\pm$ 0.02 & 0.994 $\pm$ 0.002\\
Poisson NMF & 6.05 $\pm$ 3.04 & 0.76 $\pm$ 0.05 & 0.50 $\pm$ 0.09 & 10.51 $\pm$ 13.51 & 0.73 $\pm$ 0.04 & 0.70 $\pm$ 0.05 & 3.44 $\pm$ 1.99 & 0.93 $\pm$ 0.02 & 0.995 $\pm$ 0.002\\
N-mixture & 4.75 $\pm$ 1.30 & 0.71 $\pm$ 0.06 & 0.44 $\pm$ 0.09 & 6.48 $\pm$ 4.51 & 0.72 $\pm$ 0.05 & 0.69 $\pm$ 0.05 & 1.73 $\pm$ 0.36 & 0.93 $\pm$ 0.02 & 0.995 $\pm$ 0.002\\
\textcolor{black}{UncleTC} & \textcolor{black}{17.18 $\pm$ 68.04} & \textcolor{black}{0.65 $\pm$ 0.13} & \textcolor{black}{0.41 $\pm$ 0.16} & \textcolor{black}{9.27 $\pm$ 13.19} & \textcolor{black}{0.71 $\pm$ 0.07} & \textcolor{black}{0.70 $\pm$ 0.08} & \textcolor{black}{1.79 $\pm$ 0.40} & \textcolor{black}{0.94 $\pm$ 0.02} & \textcolor{black}{0.995 $\pm$ 0.002}\\
\bottomrule
\end{tabular}
\end{table*}
We evaluate recovery performance using the relative root mean squared error (rRMSE) for count prediction accuracy, the area under the receiver operating characteristic curve (AUROC), and the area under the precision–recall curve (AUPRC) for link recovery on observed entries.  Results are summarized \textcolor{black}{as means $\pm$ standard deviations across splits.}

\textcolor{black}{Table~\ref{tab:real} shows that the proposed method consistently recovers links across datasets. 
Sparse NMF and UncleTC have the two highest AUROCs on Soundscape, but the two lowest on HPI, and UncleTC is  unstable. Robust PCA has the lowest rRMSE on all three datasets, but it returns an estimate of the observed counts without latent factors, detection probabilities, or graphs. The proposed method has a lower rRMSE than Poisson NMF on all three datasets and than sparse NMF on HPI and PPI. N-mixture and UncleTC have lower rRMSE than the proposed method on Soundscape but higher in the other two datasets. Overall, the proposed method gives link recovery comparable to the baselines while additionally returning detection estimates and within- and cross-group graphs.}

\section{Conclusion}
We study a count-based sparse bipartite network model with imperfect detection and develop an ADMM-based approach to the associated nonconvex, nonsmooth optimization problem. The structured sparsity regularization promotes controlled sparsity in the induced within-group and cross-group relations, thereby improving the recovery of similarity and connectivity. 
Under standard boundedness and inexactness assumptions, the analysis establishes asymptotic feasibility and stationarity of cluster points. {\color{black}We also characterize graph recoverability under general detection models, with additional identification conditions and finite-sample guarantees for linear detection.
The benefits of nonsmooth and nonconvex sparsity regularization on cross- and within-group graphs come at a higher computational cost than simpler factorization baselines in the literature. Future work will focus on improving computational efficiency, relaxing assumptions about detection and count distributions while retaining theoretical guarantees, and validating the inferred graphs against independent ecological observations.}

\appendices
\section{Proofs of Recoverability (Section~\ref{sec:ident})}\label{app:recov}
{\color{black}
\emph{Overview.} Lemma~\ref{lem:mean} shows that every global minimizer fits the true mean; Theorem~\ref{thm:finite} converts this fit into recovery of the edge sets; and Appendix~\ref{app:ident} gives conditions for the identification hypothesis of part~(b).

\subsection{Mean-Fit Bound}\label{app:meanfit}
We measure the fit of a mean vector $\boldsymbol\nu$, e.g., $\boldsymbol\nu=\boldsymbol\mu(\boldsymbol\zeta)$, by
\[
H(\boldsymbol\nu)=\textstyle\sum_{(i,j)\in\Omega}\big(\sqrt{\nu_{ij}+\delta_E}-\sqrt{\mu_{0,ij}}\big)^2 .
\]
\begin{lemma}\label{lem:mean}
Let $\boldsymbol\zeta_0\in\mathcal C$, $\mathbf Z$ have full column rank, and $t>0$. With probability at least $1-2e^{-t}$, every global minimizer $\widehat{\boldsymbol\zeta}$ of~\eqref{eq:mle} with $\widehat{\boldsymbol\mu}=\boldsymbol\mu(\widehat{\boldsymbol\zeta})$ satisfies $H(\widehat{\boldsymbol\mu})\le|\Omega|\varepsilon_E$, where
\begin{equation}\label{eq:epsE}
\varepsilon_E=\frac{8d\,\ell_E+8t+4\,\mathcal R(\boldsymbol\zeta_0)+7}{E|\Omega|}, \quad \ell_E=\log(1+ 6\bar\mu/\eta),
\end{equation}
with $\bar\mu=FB^2$, and $\eta=\big(2E|\Omega|(2E|\Omega|\bar\mu+t+1)\big)^{-1}$.
\end{lemma}
\begin{proof}
We follow the standard argument for penalized likelihood~\cite{li1999mixture,soni2016noisy}. 
For a mean vector $\boldsymbol\nu$, let $Q_{\boldsymbol\nu}(\mathbf y)$ denote the probability of observing counts $\mathbf y$ under independent observations $\mathbf{Y}_{ij}\sim\mathrm{Poisson}\big(E(\nu_{ij}+\delta_E)\big)$, $(i,j)\in\Omega$, and let $\mathbb P_0(\mathbf y)$ denote the corresponding probability under the true means.  Since $-\log Q_{\boldsymbol\mu(\boldsymbol\zeta)}(\mathbf Y)$ differs from $-L_{\delta_E}(\boldsymbol\zeta)$ only by a constant, $\widehat{\boldsymbol\zeta}$ also minimizes $-\log Q_{\boldsymbol\mu(\boldsymbol\zeta)}(\mathbf Y)+\mathcal R(\boldsymbol\zeta)$ over $\mathcal C$.

\emph{Step 1:} Since $\sum_y\sqrt{\mathrm{Pois}(y;a)\,\mathrm{Pois}(y;b)}=e^{-(\sqrt a-\sqrt b)^2/2}$, we have $\mathbb E_0\sqrt{Q_{\boldsymbol\nu}/\mathbb P_0}=e^{-EH(\boldsymbol\nu)/2}$. By Markov's inequality, for every $s$,
\[
\mathbb P_0\big(\log(Q_{\boldsymbol\nu}/\mathbb P_0)\ge-2s\big)\le e^{\,s-EH(\boldsymbol\nu)/2}.
\]

\emph{Step 2.} If $\bar\mu=0$, every feasible mean is zero and $H(\widehat{\boldsymbol\mu})=1/E$. Otherwise, set $h=\eta/(3\bar\mu)$ and normalize the blocks as $\mathbf A=\mathbf U/B$, $\mathbf D=\mathbf V/B$, and $\mathbf p=\mathbf Z_\Omega\boldsymbol\alpha$. $\mathbf A$, $\mathbf D$, and $\mathbf p$ lie in the in maximum-entry unit balls of dimensions $IF$, $JF$, and $r=\mathrm{rank}(\mathbf Z_\Omega)\le R$. In each block's feasible set, take a maximal $h$-separated set; this set is an $h$-net. In dimension $m$, comparing the volumes of its disjoint balls of radius $h/2$ with the containing ball of radius $1+h/2$ gives at most $(1+2/h)^m$ points.

Let $\mathcal G$ contain the mean vectors $\nu_{ij}=B^2\mathbf p'_{ij}\sum_f \mathbf A'_{if}\mathbf D'_{jf}$ obtained from the three nets. For each $\boldsymbol\zeta\in\mathcal C$, choose net points within $h$ of its three blocks. Since all entries lie in $[0,1]$, changing each block in turn contributes at most $FB^2h$ to the mean error. Thus the resulting $\boldsymbol\nu\in\mathcal G$ satisfies
$\|\boldsymbol\mu(\boldsymbol\zeta)-\boldsymbol\nu\|_\infty
\le 3FB^2h=\eta$
for a suitable $\boldsymbol\nu\in\mathcal G$ for every $\boldsymbol\zeta\in\mathcal C$. Multiplying the three net sizes gives
\[
\log|\mathcal G|\le(F(I+J)+r)\log(1+2/h)\le d\,\ell_E.
\]
Step~1 with $s=EH(\boldsymbol\nu)/4-\log|\mathcal G|-t$ and a union bound over $\mathcal G$ give, with probability at least $1-e^{-t}$,    
\begin{equation}\label{eq:netbound}
\log\frac{Q_{\boldsymbol\nu}}{\mathbb P_0}<-\frac E2H(\boldsymbol\nu)+2d\,\ell_E+2t\qquad\text{for all }\boldsymbol\nu\in\mathcal G.
\end{equation}

\emph{Step 3.} Let $\widehat{\boldsymbol\nu}\in\mathcal G$ be within $\eta$ of $\widehat{\boldsymbol\mu}$. Since $|\log(a+\delta_E)-\log(b+\delta_E)|\le|a-b|/\delta_E$, we have $|\log Q_{\widehat{\boldsymbol\mu}}-\log Q_{\widehat{\boldsymbol\nu}}|\le(S/\delta_E+E|\Omega|)\eta$, where $S=\sum_\Omega \mathbf Y_{ij}$ is Poisson with mean $\leq E|\Omega|\bar\mu$. Since $\mathbb P(S\ge 2\mathbb E(S)+t)\le e^{-t}$, the choice of $\eta$ makes this difference $\leq 1/2$ with probability at least $1-e^{-t}$.

\emph{Step 4.} Since $\widehat{\boldsymbol\zeta}$ minimizes~\eqref{eq:mle} and $\mathcal R\ge0$, $\log Q_{\widehat{\boldsymbol\mu}}\ge\log Q_{\boldsymbol\mu_0}-\mathcal R(\boldsymbol\zeta_0)$. Moreover, $\log(Q_{\boldsymbol\mu_0}/\mathbb P_0)\ge-E|\Omega|\delta_E=-1$, since each cell contributes at least $-E\delta_E$, and $\mathbf Y_{ij}=0$ almost surely wherever $\mu_{0,ij}=0$.

 On both events, Steps~3 and~4 give $\log(Q_{\widehat{\boldsymbol\nu}}/\mathbb P_0)\ge-\mathcal R(\boldsymbol\zeta_0)-3/2$, and~\eqref{eq:netbound} then gives $EH(\widehat{\boldsymbol\nu})<4d\ell_E+4t+2\mathcal R(\boldsymbol\zeta_0)+3$. Since $(\sqrt{a+\delta_E}-\sqrt{b+\delta_E})^2\le|a-b|$, we obtain $H(\widehat{\boldsymbol\mu})\le2H(\widehat{\boldsymbol\nu})+2|\Omega|\eta\le|\Omega|\varepsilon_E$ with $2|\Omega|\eta\le1/E$.
\end{proof}

\subsection{Proof of Theorem~\ref{thm:finite}}\label{app:proofs:ident}
\begin{proof}
\textit{(a)} Each misclassified cell uses a fixed part of the error budget $H(\widehat{\boldsymbol\mu})\le|\Omega|\varepsilon_E$, so there can be few of them. Let $\underline\mu=p_{\min}\beta$. Since $\mathbf P_0\ge p_{\min}$ on $\mathcal C$, every link has $\mu_{0,ij}\ge\underline\mu$. A misclassified cell is of one of two kinds:
\begin{itemize}[leftmargin=*]
\item a \emph{false link}, with $\mu_{0,ij}=0$ but $\widehat\mu_{ij}+\delta_E\ge\underline\mu/2$. Its term in $H$ is $\widehat\mu_{ij}+\delta_E\ge\underline\mu/2$;
\item a \emph{missed link}, with $\mu_{0,ij}\ge\underline\mu$ but $\widehat\mu_{ij}+\delta_E<\underline\mu/2\le\mu_{0,ij}/2$. Its term is larger than $(\sqrt{\mu_{0,ij}}-\sqrt{\mu_{0,ij}/2})^2=c_0\mu_{0,ij}\ge c_0\underline\mu$, where $c_0=(1-1/\sqrt2)^2$.
\end{itemize}
Since $c_0<1/2$, every misclassified cell adds more than $c_0\underline\mu$ to $H$. Thus $|\widehat{\mathcal E}\triangle(\mathcal E_0\cap\Omega)|< |\Omega|\varepsilon_E/(c_0\underline\mu)$. 
If $|\Omega|\varepsilon_E<c_0\underline\mu$, the budget is smaller than the cost of a single mistake and no cell is misclassified. 
Hence, every link has a larger fitted mean than every non-link.

\textit{(b)} Let $\mathcal O_0=\{\boldsymbol\zeta\in\mathcal C:\mu_{ij}(\boldsymbol\zeta)=\mu_{0,ij}, \forall(i,j)\in\Omega\}$ be the exact fits of the mean. By the identification hypothesis, every $\boldsymbol\zeta\in\mathcal O_0$ has the true edge sets. 
Let $\mathrm{dist}(\boldsymbol\zeta,\mathcal O_0)=\min_{\boldsymbol\zeta'\in\mathcal O_0}\max_X\|\mathbf M_X(\boldsymbol\zeta)-\mathbf M_X(\boldsymbol\zeta')\|_{\max}$ and 
$D(\boldsymbol\zeta) :=
\frac1{|\Omega|} \sum_{(i,j)\in\Omega} \big(\sqrt{\mu_{ij}(\boldsymbol\zeta)} -\sqrt{\mu_{0,ij}}\big)^2$ be the loss of fit of $\boldsymbol\zeta$. 
We define the margins
\begin{equation}\label{eq:margins}
\begin{aligned}
& w_{\min}:=\min_{\boldsymbol\zeta\in\mathcal O_0}\ \min_X\ \min_{(i,k)\in\mathcal E(\mathbf M_X(\boldsymbol\zeta_0))}\big(\mathbf M_X(\boldsymbol\zeta)\big)_{ik},\\
& \kappa:=\inf_{\boldsymbol\zeta\in\mathcal C}\Big\{D(\boldsymbol{\zeta}):\mathrm{dist}(\boldsymbol\zeta,\mathcal O_0)\ge w_{\min}/3\Big\}.
\end{aligned}
\end{equation}
$\mathcal C$ is closed and bounded because $\|\mathbf Z\boldsymbol\alpha\|_\infty\le1$ and $\mathbf Z$ has full column rank. The maps $\boldsymbol\zeta\mapsto\boldsymbol\mu(\boldsymbol\zeta)$ and $\boldsymbol\zeta\mapsto\mathbf M_X(\boldsymbol\zeta)$ are continuous, so $\mathcal O_0\ni\boldsymbol\zeta_0$ is compact. 
The inner minimum in $w_{\min}$ is continuous and, by the identification hypothesis, positive on $\mathcal O_0$, so $w_{\min}>0$. The set $\{\boldsymbol\zeta\in\mathcal C:\mathrm{dist}(\boldsymbol\zeta,\mathcal O_0)\ge w_{\min}/3\}$ is compact and contains no exact fit, so its continuous loss of fit has a positive minimum, i.e., $\kappa>0$.

Since $(\sqrt{\widehat\mu}-\sqrt{\mu_0})^2\le2(\sqrt{\widehat\mu+\delta_E}-\sqrt{\mu_0})^2+2\big(\sqrt{\widehat\mu+\delta_E}-\sqrt{\widehat\mu}\big)^2$, $2\big(\sqrt{\widehat\mu+\delta_E}-\sqrt{\widehat\mu}\big)^2 \leq 2\delta_E\le\varepsilon_E$ and 
\[D(\widehat{\boldsymbol\zeta})
\le 
2H(\widehat{\boldsymbol\mu})/|\Omega|+2\delta_E
\le2\varepsilon_E+\varepsilon_E
=3\varepsilon_E < \kappa.\]
By the definition of $\kappa$, some $\boldsymbol\zeta'\in\mathcal O_0$ has $\max_X\|\mathbf M_X(\widehat{\boldsymbol\zeta})-\mathbf M_X(\boldsymbol\zeta')\|_{\max}<w_{\min}/3$, i.e., $\widehat{\boldsymbol\zeta}$ is close to an exact fit.

Each $\mathbf M_X(\boldsymbol\zeta')$ is at least $w_{\min}$ on the true edges and zero elsewhere. Hence $\mathbf M_X(\widehat{\boldsymbol\zeta})$ exceeds $2w_{\min}/3$ on the true edges and is below $w_{\min}/3$ elsewhere, and thresholding at $w_{\min}/2$ recovers all three edge sets.
\end{proof}

\subsection{When the Graphs Are Identified}\label{app:ident}
Assume full observation and use identification as in Definition~\ref{def:ident}. Every exact fit with ${\mathbf{P}}>0$ satisfies $\mathcal E({\mathbf{U}}{\mathbf{V}}^\top)=\mathcal E_0$, so the cross-group graph is always identified.

With unrestricted detection, take any nonnegative $(\boldsymbol\Phi,\boldsymbol\Psi)$ with $\mathcal E(\boldsymbol\Phi\boldsymbol\Psi^\top)=\mathcal E_0$. For large enough $c$, $({\mathbf{U}},{\mathbf{V}})=(c\boldsymbol\Phi,\boldsymbol\Psi)$ with $\mathbf P_{ij}=\mu_{0,ij}/(c(\boldsymbol\Phi\boldsymbol\Psi^\top)_{ij})$ on $\mathcal E_0$ and $\mathbf P_{ij}=1$ elsewhere. Thus, within-group supports need not be identified. For example, if $\mathcal E_0$ is complete but $\mathcal E(\mathbf U_0\mathbf U_0^\top)$ is not, choosing strictly positive $\boldsymbol\Phi,\boldsymbol\Psi$ in the construction above gives an exact fit with a complete within-group graph $\mathcal E(\boldsymbol\Phi\boldsymbol\Phi^\top)$.

For nonnegative factors, $(\boldsymbol\Phi\boldsymbol\Psi^\top)_{ij}>0$ if and only if $\boldsymbol\Phi_{if}\boldsymbol\Psi_{jf}>0$ for some $f$. The edge set of the $f$-th rank-one term is therefore a rectangle $\mathcal A_f\times\mathcal B_f\subseteq\mathcal E(\boldsymbol\Phi\boldsymbol\Psi^\top)$, where $\mathcal A_f$ and $\mathcal B_f$ are the supports of the $f$-th columns. Let $\mathrm{rect}(\mathcal E_0)$ be the largest $|\mathcal A|+|\mathcal B|$ over nonempty $\mathcal A\times\mathcal B\subseteq\mathcal E_0$.

\begin{proposition}\label{prop:generic}
Let $g$ be linear, fix ${\mathbf{U}}_0,{\mathbf{V}}_0,\boldsymbol\alpha_0$, and let the features $\{{\mathbf{Z}}_{ij}\}$ have a joint density on $\mathbb R^{IJR}$, or on a product of simplices, supported on $\{{\mathbf{Z}}\boldsymbol\alpha_0>0\}$. If
\begin{equation}\label{eq:count}
|\mathcal E_0|\ge R+F(\mathrm{rect}(\mathcal E_0)-1),
\end{equation}
then almost surely every parameter fitting the true mean with ${\mathbf{P}}>0$, allowing zero factor columns, has $\boldsymbol\alpha=c\boldsymbol\alpha_0$ for some $c>0$. Hence ${\mathbf{P}}=c{\mathbf{P}}_0$ and ${\mathbf{U}}{\mathbf{V}}^\top=c^{-1}\boldsymbol\Lambda_0$.
\end{proposition}
Condition~\eqref{eq:count} ensures that the number of link equations exceeds the dimension bound for alternatives.
\begin{proof}
A parameter fits the true mean if $({\mathbf{Z}}_{ij}^\top\boldsymbol\alpha)\lambda_{ij}=({\mathbf{Z}}_{ij}^\top\boldsymbol\alpha_0)\lambda_{0,ij}$ for all $(i,j)$. These equations involve only $\boldsymbol\alpha$ and $\boldsymbol\Lambda$ and hold if $\boldsymbol\alpha$ is multiplied by some $c>0$ and $\boldsymbol\Lambda$ is divided by $c$, so we may assume $\|\boldsymbol\alpha\|=1$. The truth solves them with $\boldsymbol\alpha\in\mathrm{span}(\boldsymbol\alpha_0)$. We call a solution with $\boldsymbol\alpha\notin\mathrm{span}(\boldsymbol\alpha_0)$ an alternative.
Since $\mathcal E(\boldsymbol\Lambda)=\mathcal E_0$, each nonzero rank-one term has at most $\mathrm{rect}(\mathcal E_0)-1$ free parameters after fixing its column scale. There are at most $F$ such terms and finitely many support patterns, so the alternatives form a semialgebraic set of dimension at most $(R-1)+F(\mathrm{rect}(\mathcal E_0)-1)$.

For a fixed alternative, fitting the true mean requires ${\mathbf{Z}}_{ij}^\top(\lambda_{ij}\boldsymbol\alpha-\lambda_{0,ij}\boldsymbol\alpha_0)=0$ with $\lambda_{ij}\boldsymbol\alpha \neq \lambda_{0,ij}\boldsymbol\alpha_0$ for all entries. Each equation restricts one feature block to a hyperplane. Since different equations involve disjoint blocks, each removes 1 dimension, giving total codimension $|\mathcal E_0|$. On a simplex, each condition is affine and either nontrivial or infeasible.

Let $D$ be the dimension of the feature space. The features that admit an alternative form the projection of a semialgebraic set of dimension at most $D-|\mathcal E_0|+(R-1)+F(\mathrm{rect}(\mathcal E_0)-1)$, which is less than $D$ by~\eqref{eq:count}. A projection does not increase the dimension, so this set has measure zero. Otherwise $\boldsymbol\alpha=c\boldsymbol\alpha_0$, with $c>0$ because ${\mathbf{Z}}\boldsymbol\alpha>0$ and ${\mathbf{Z}}\boldsymbol\alpha_0>0$.
\end{proof}

With equations on $\Omega$ only, the same dimension argument applies after replacing $|\mathcal E_0|$ on the left of~\eqref{eq:count} by $|\mathcal E_0\cap\Omega|$, and replacing $\mathcal E_0$ inside $\mathrm{rect}(\cdot)$ by $\mathcal E_0\cup\Omega^c$. Under this modified condition, almost surely every fit of the observed means with positive detection and no zero factor columns has $\boldsymbol\alpha=c\boldsymbol\alpha_0$ and hence $\mathbf P=c\mathbf P_0$ on all cells, but the equations give $\lambda_{ij}=c^{-1}\lambda_{0,ij}$ only for $(i,j)\in\Omega$. Identification of the full graphs also requires that all feasible factorizations of admissible completions share the same edge set.

\begin{corollary}\label{cor:supp}
Under the conditions of Proposition~\ref{prop:generic}, suppose that all nonnegative rank-$F$ factorizations of $\boldsymbol\Lambda_0$ without zero columns induce the same $\mathcal E({\mathbf{U}}{\mathbf{U}}^\top)$ and $\mathcal E({\mathbf{V}}{\mathbf{V}}^\top)$, as holds, e.g., when ${\mathbf{U}}_0$ and ${\mathbf{V}}_0$ are sufficiently scattered~\cite{huang2014nmf,fu2019nmf}. Then all three edge sets are identified almost surely. If, in addition, $\mathrm{rank}(\boldsymbol\Lambda_0)=F$, the identification hypothesis of Theorem~\ref{thm:finite}(b) holds almost surely.
\end{corollary}
\begin{proof}
By Proposition~\ref{prop:generic}, $(c^{1/2}{\mathbf{U}},c^{1/2}{\mathbf{V}})$ is a nonnegative rank-$F$ factorization of $\boldsymbol\Lambda_0$ without zero columns, so it induces the true edge sets, and so does $({\mathbf{U}},{\mathbf{V}})$. Under sufficient scattering, the factorization is unique up to column permutation and positive scaling~\cite{huang2014nmf,fu2019nmf}, neither of which changes an edge set.
Let $\boldsymbol\zeta\in\mathcal O_0$, so that ${\mathbf{P}}\ge p_{\min}>0$, but ${\mathbf{U}}$ or ${\mathbf{V}}$ may have zero columns. By Proposition~\ref{prop:generic}, ${\mathbf{U}}{\mathbf{V}}^\top=c^{-1}\boldsymbol\Lambda_0$. If $\mathrm{rank}(\boldsymbol\Lambda_0)=F$, neither factor has a zero column. Hence $\boldsymbol\zeta$ is equivalent, and the preceding argument applies.
\end{proof}
}

\section{Proofs of Section \ref{sec:converg:analysis}}\label{app:proof_conv}
\begin{proof}[a) Proof for Lemma \ref{lem:exact_a}]
 Since the $\mathbf A_X$-update is separable across entries, it suffices to study the scalar problem
\begin{align*}
    \min_{a\in\mathbb R}\;
g(a):=\lambda \|a\|_{1/2}+({\rho}/{2})(a-b)^2,
\end{align*}
where $\lambda=\lambda_X$, $\rho=\rho_X^k$, and $b$ is one entry of $\mathbf M_X^{k+1}+\mathbf W_X^k$.
{\color{black}
Let $a^\star$ be the solution. 
If $a^\star\neq 0$, $a^\star$ is a nonzero minimizer and thus is a stationary point of $g$, i.e.,
$0=({\lambda}/{2}){\operatorname{sign}(a^\star)}/{\sqrt{|a^\star|}}+\rho(a^\star-b)$.}
By symmetry, it is enough to treat the case $a^\star>0$ and then $
b=a^\star+{\lambda}/{(2\rho\sqrt{a^\star})}$. 
At stationary, comparison with $0$ gives $g(a^\star ) - g(0) = ({\sqrt{a^\star}}/{2})(\lambda - \rho a^{3/2})$. {\color{black}For a positive stationary point to be  the global minimizer, we need $g(a^\star) < g(0) $ where}
\begin{align*}
&g(a^\star)-g(0)
=
\lambda \sqrt{a^\star}
+\frac{\rho}{2}(a^\star-b)^2
-\frac{\rho}{2}b^2 \\
&{=}
\lambda \sqrt{a^\star}
{+}\frac{\lambda^2}{8\rho a^\star}
{-}\frac{\rho}{2}\left(a^\star+\frac{\lambda}{2\rho\sqrt{a^\star}}\right)^2 
{=}
\frac{\sqrt{a^\star}}{2}\Bigl(\lambda-\rho (a^\star)^{3/2}\Bigr).
\end{align*}
Thus $g(a^\star) < g(0)$ implies $a^\star > (\lambda/\rho)^{2/3}$ and $b > \tau := ({3}/{2})(\lambda/\rho)^{2/3}$. 
{\color{black}If $a^\star = 0$, then $|b|\le\tau:=\frac32(\lambda/\rho)^{2/3}$.}
Along with the half-thresholding characterization~\cite{xu2012lhalf}, we obtain the global minimizer by the entrywise half-thresholding operator.
\end{proof}

\begin{proof}[b) Proof for Lemma \ref{lem:bound_y}]
By Lemma~\ref{lem:exact_a}, for entrywise $\mathbf A_X$-update, we have for $|b| \leq \tau$, $a^\star = 0$ and $|b-a^\star| \leq \tau = ({3}/{2})(\lambda/\rho)^{2/3}$. For $|b| > \tau$, $a^\star > (\lambda/\rho)^{2/3}$ and
 \vspace{-0.5em}
\[
|b-a^\star|
=
\frac{\lambda}{2\rho\sqrt{a^\star}}
\le
\frac{\lambda}{2\rho(\lambda/\rho)^{1/3}}
=
\frac12\Bigl(\frac{\lambda}{\rho}\Bigr)^{2/3}.
\]

Therefore, for every entry of $\mathbf A_X^{k+1}$, for all $a^\star$'s,
\textcolor{black}{$|(\mathbf M_X^{k+1}+\mathbf W_X^k)_{ij}-(\mathbf A_X^{k+1})_{ij}|
\le ({3}/{2})\big({\lambda_X}/{\rho_X^k}\big)^{2/3}$.}
Summing over all $n_X$ entries gives
{\color{black}\begin{align*}
\|\mathbf M_X^{k+1}+\mathbf W_X^k-\mathbf A_X^{k+1}\|_F
& \le
({3}/{2})\sqrt{n_X}\big({\lambda_X}/{\rho_X^k}\big)^{2/3}\\
& = C_X(\rho_X^k)^{-2/3},
\end{align*}}
 \vspace{-0.75em}

where $C_X = \frac{3}{2}\sqrt{n_X}\lambda_X^{2/3}$. Since
\textcolor{black}{$\mathbf M_X^{k+1}-\mathbf A_X^{k+1} = r_X^{k+1}$,}
we obtain $
\|\mathbf W_X^k+r_X^{k+1}\|_F\le C_X(\rho_X^k)^{-2/3}$.

Next, for the scaled-dual update, by Algorithm~\ref{alg:merged_soundscape_code},
$\mathbf W_X^{k+1}
= ({\rho_X^k}/{\rho_X^{k+1}})(\mathbf W_X^k+r_X^{k+1})$
with $\rho_X^{k}/\rho_X^{k+1} \leq 1$.
Hence
\begin{align}\label{eq:dual_bound}
    \|\mathbf W_X^{k+1}\|_F
{\le} \|\mathbf W_X^k{+}r_X^{k+1}\|_F {\le}
C_X(\rho_X^k)^{-2/3}.
\end{align}
For $k \geq 1$,
\begin{equation}\label{eq:bound_y_r_up}
    \|r_X^{k+1}\|_F
\le \|\mathbf W_X^k+r_X^{k+1}\|_F+\|\mathbf W_X^k\|_F
\le 2C_X(\rho_X^k)^{-2/3}.
\end{equation}

We now show that $\rho_X^k$ can be increased only finitely many times.
Suppose, for contradiction, that the penalty-update test $\|r_X^{k+1}\|_F>\epsilon_X^k$ is triggered infinitely often. Then $\rho_X^k\to\infty$ along those trigger times. {\color{black} By Assumption \ref{assum:err_con}, for all sufficiently large $k$, $\epsilon_X^k \ge 2C_X(\rho_X^k)^{-2/3}$.} Together with (\ref{eq:bound_y_r_up}),
\[
\|r_X^{k+1}\|_F\le 2C_X(\rho_X^k)^{-2/3}\le \epsilon_X^k,
\]
contradicting the trigger condition. 
Hence there exists $K_X$ and $\bar\rho_X>0$ such that $\forall k\ge K_X$, $\rho_X^k=\bar\rho_X$. By the update rule, once the penalty stops increasing, $\|r_X^{k+1}\|_F\le \epsilon_X^k$ must hold for every $k\ge K_X$, proving the first claim. 
\textcolor{black}{The bound \eqref{eq:dual_bound} holds for every $k$, so $\{\mathbf W_X^k\}$ is bounded.}
\end{proof}

\begin{proof}[c) Proof for Lemma \ref{lem:bound_L}]
By Lemma~\ref{lem:bound_y}, \textcolor{black}{for arbitrary $k$}, 
\textcolor{black}{$\{\mathbf W_X^k\}$ is bounded and $\rho_X^0\le\rho_X^k\le\bar\rho_X$, since the penalty is nondecreasing with finite increase.}
\begin{multline*}
    \textcolor{black}{\frac{\rho_X^k}{2}\Bigl(\lVert r_X^k + \mathbf W_X^k \rVert_F^2 - \lVert \mathbf W_X^k\rVert_F^2\Bigr)} \\
    \textcolor{black}{\geq -\frac{\rho_X^k}{2}\lVert \mathbf W_X^k\rVert_F^2 
    \geq - \frac{\bar\rho_X}{2} \sup_{k}\lVert  \mathbf W_X^k\rVert_F^2} > -\infty.
\end{multline*}  

     By Assumption~\ref{assum:bounded-iterates}, $\mathbf U^k, \mathbf V^k$ are bounded, i.e., $\exists \bar{\Delta} > 0$ such that $x_{i,j}^k = {\mathbf u_i^k}^\top \mathbf v_j^k \leq \lVert \mathbf u_i^k\rVert\lVert \mathbf v_j^k\rVert \leq \bar \Delta$.
     By Assumption~\ref{assum:block_lipschitz}, along the iterates, for any \textcolor{black}{$\mathbf Y_{ij}>0$}, $\exists \delta > 0$ such that $x_{i,j}^k \geq \delta$. Denote 
     $h(x_{i,j}^k) := \textcolor{black}{g(\mathbf{Z}_{ij};\boldsymbol{\alpha}^k)}x_{i,j}^k - \textcolor{black}{\mathbf  Y_{ij}}\log x_{i,j}^k  $. If \textcolor{black}{$\mathbf Y_{ij}> 0$}, $h(x_{i,j}^k)$ is continuous for $x_{i,j}^k \in [\delta, \bar\Delta]$. If \textcolor{black}{$\mathbf Y_{ij} = 0, h(x_{i,j}^k) \geq 0$} and is bounded below entrywise. Summing over all entries, $\inf_k f(\mathbf U^k, \mathbf V^k) > -\infty$. As $\lambda_X > 0$, $\lambda_X\lVert \mathbf A_X^k\rVert_{1/2} \geq 0$ and
       $  \mathcal L_{\rho^k}(\mathbf U,\mathbf V,\mathbf A,\textcolor{black}{\mathbf W})
          \geq  \inf_k f(\mathbf U^k, \mathbf V^k)+  \sum_{X}\left(\lambda_X\lVert \mathbf A_X^k\rVert_{1/2}\right) -\sum_{X}\textcolor{black}{\left((\bar\rho_X/2)\sup_k \|\mathbf W_X^k\|_F^2 \right)}
           > -\infty$.
\end{proof}

\begin{proof}[d) Proof for Lemma \ref{lem:alm_desc}]
We split the proof into the $\mathbf U$-update, the $\mathbf V$-update, the exact $\mathbf A$-update, and the dual/penalty update.

Fix $k$ and $\mathbf V^k$. The $\mathbf U$-subproblem minimizes the smooth function $F_k(\mathbf U,\mathbf V^k)$ over the nonnegative orthant.
Let
\begin{align*}
\mathbf U_{s+1}^k=\mathcal{P}_{\mathbb R_+}\!\bigl(\mathbf U_s^k-t_U^{s,k}G_s^k\bigr),
\qquad
G_s^k:=\nabla_{\mathbf U} F_k(\mathbf U_s^k,\mathbf V^k),
\end{align*}
where $\mathcal{P}_{\mathbb R_+}$ denotes Euclidean projection onto the nonnegative orthant. 
\textcolor{black}{By the Euclidean projection characterization,}
\begin{align*}
    \left\langle \mathbf U_s^k-t_U^{s,k}G_s^k-\mathbf U_{s+1}^k,\; \mathbf U-\mathbf U_{s+1}^k\right\rangle \le 0
\qquad \forall \mathbf U\ge \mathbf  0.
\end{align*}
Taking $\mathbf U=\mathbf U_s^k$,
$\langle G_s^k, \mathbf U_{s+1}^k-\mathbf U_s^k\rangle
\le
-\frac{1}{t_U^{s,k}}\|\mathbf U_{s+1}^k-\mathbf U_s^k\|_F^2$.

By Assumption~\ref{assum:block_lipschitz}, for the observed cells and \textcolor{black}{the curvature bound on the floored missing-cell terms, $F_k(\cdot,{\mathbf{V}}^k)$ satisfies the descent lemma on the inner loop with $L_U^k \le \bar L_U$. Thus,}
\begin{multline*}
    F_k(\mathbf U_{s+1}^k,\mathbf V^k)
\le F_k(\mathbf U_s^k,\mathbf V^k)\\
+\langle G_s^k,\mathbf U_{s+1}^k-\mathbf U_s^k\rangle
+\frac{L_{U,s}^k}{2}\|\mathbf U_{s+1}^k-\mathbf U_s^k\|_F^2,
\end{multline*}
where $L_{U,s}^k$ is a valid local Lipschitz constant.
Since the step size is chosen to satisfy
$t_U^{s,k}\le 1/L_{U,s}^k$,
\begin{equation}\label{eq:diff_F_U_inner}
F_k(\mathbf U_{s+1}^k,\mathbf V^k)
\le F_k(\mathbf U_s^k,\mathbf V^k)
-\frac{1}{2t_U^{s,k}}\|\mathbf U_{s+1}^k-\mathbf U_s^k\|_F^2.
\end{equation}

Summing \eqref{eq:diff_F_U_inner} over $s=0,\dots,S_{\max}-1$ gives
\[
F_k(\mathbf U^{k+1},\mathbf V^k)
\le F_k(\mathbf U^k,\mathbf V^k)
-\sum_{s=0}^{S_{\max} - 1}\frac{1}{2t_U^{s,k}}\|\mathbf U_{s+1}^k-\mathbf U_s^k\|_F^2.
\]
By Jensen's inequality,
 \vspace{-0.75em}
\begin{multline*}
  \sum_{s=0}^{S_{\max}-1}\|\mathbf U_{s+1}^k-\mathbf U_s^k\|_F^2
\ge
\frac{1}{S_{\max}}
\left\|\sum_{s=0}^{S_{\max}-1}(\mathbf U_{s+1}^k-\mathbf U_s^k)\right\|_F^2\\
=
\frac{1}{S_{\max}}\|\mathbf U^{k+1}-\mathbf U^k\|_F^2.  
\end{multline*}
\vspace{-0.75em}

By Assumption~\ref{assum:block_lipschitz}, $\forall k > 0$, $L_U^k < \infty$, so $t_U^k$ is upper bounded. Define $\bar t_U^k := \max_{s \geq 0} t_U^{s,k}$. Since $t_U^{s,k}\le \bar t_U^k$ for all $s$,
 \vspace{-0.35em}
\begin{align*}
    F_k(\mathbf U^{k+1},\mathbf V^k) \le F_k(\mathbf U^k,\mathbf V^k)
-\frac{1}{2\bar t_U^k}\sum_{s=0}^{S_{\max}-1}\|\mathbf U_{s+1}^k-\mathbf U_s^k\|_F^2.
\end{align*}
\vspace{-0.75em}

Therefore
 \vspace{-0.35em}
\begin{equation}\label{eq:diff_u_inner_final}
   F_k(\mathbf U^{k+1},\mathbf V^k) \le F_k(\mathbf U^k,\mathbf V^k)-
   \textcolor{black}{c_U^k}\|\Delta \mathbf U^k\|_F^2, 
\end{equation}
\vspace{-0.75em}

where $\textcolor{black}{c_U^k} := \frac{1}{2\bar t_U^k S_{\max}}$. Similarly, with $\mathbf U^{k+1}$ fixed,
\begin{equation}\label{eq:diff_v_inner_final}
   F_k(\mathbf U^{k+1},\mathbf V^{k+1})
\le
F_k(\mathbf U^{k+1},\mathbf V^k)-c_V^k\|\Delta \mathbf V^k\|_F^2,
\end{equation}
where $c_V^k:=\frac{1}{2\bar t_V^k S_{\max}}$.
Combining (\ref{eq:diff_u_inner_final}) and (\ref{eq:diff_v_inner_final}),
\begin{equation}\label{eq:u_v_descent}
  F_k(\mathbf U^{k+1},\mathbf V^{k+1}) \le F_k(\mathbf U^k,\mathbf V^k)-c_U^k\|\Delta \mathbf U^k\|_F^2-c_V^k\|\Delta \mathbf V^k\|_F^2.  
\end{equation}
\textcolor{black}{Between outer iterations, the imputed counts change from $\widehat{\mathbf Y}^k$ to $\widehat{\mathbf Y}^{k+1}$, so $F_k$ depends on $k$. By Remark~\ref{rem:impute}, $f_\Omega(\mathbf U,\mathbf V)-f_\Omega(\mathbf U^k,\mathbf V^k)\le F_k(\mathbf U,\mathbf V)-F_k(\mathbf U^k,\mathbf V^k)$, so~\eqref{eq:u_v_descent} also holds with $f$ replaced by the observed-data objective $f_\Omega$ where the imputed counts do not appear. We use this form below so the imputation does not enter the telescoping sum.}

With $\mathbf A^k$ and \textcolor{black}{$\mathbf W^k$} fixed, 
$\mathcal L_{\rho^k}(\mathbf U,\mathbf V,\mathbf A^k,\textcolor{black}{\mathbf W^k})$ differs from $F_k(\mathbf U,\mathbf V)$ only by terms independent of $\mathbf U,\mathbf V$.
Hence by \eqref{eq:u_v_descent},
\begin{multline*}
    \mathcal L_{\rho^k}(\mathbf U^{k+1},\mathbf V^{k+1},\mathbf A^k,\textcolor{black}{\mathbf W^k}) \le
\mathcal L_{\rho^k}(\mathbf U^k,\mathbf V^k,\mathbf A^k,\textcolor{black}{\mathbf W^k})\\
-c_U^k\|\Delta \mathbf U^k\|_F^2-c_V^k\|\Delta \mathbf V^k\|_F^2,
\end{multline*}
where $\bar t_U^k\le t_U^{\max}$ by Assumption~\ref{assum:err_con} and \ref{assum:block_lipschitz}. Thus, for all $k$, 
\[
c_U^k=\frac{1}{2\bar t_U^k S_{\max}}\ge \frac{1}{2t_U^{\max}S_{\max}}=:c_U>0.
\]
The same argument gives $c_V^k\ge c_V>0$. For all $k$, $-c_U^k\|\Delta \mathbf U^k\|_F^2 \leq -c_U\|\Delta \mathbf U^k\|_F^2$ and similarly for $\mathbf V$ and 
\begin{multline}\label{eq:comb_desc_l}
    \mathcal L_{\rho^k}(\mathbf U^{k+1},\mathbf V^{k+1},\mathbf A^k,\textcolor{black}{\mathbf W^k})
\le
\mathcal L_{\rho^k}(\mathbf U^k,\mathbf V^k,\mathbf A^k,\textcolor{black}{\mathbf W^k})\\
-c_U\|\Delta \mathbf U^k\|_F^2-c_V\|\Delta \mathbf V^k\|_F^2.
\end{multline}

By Lemma~\ref{lem:exact_a}, each $\mathbf A_X^{k+1}$ is a global minimizer, so
\begin{equation}\label{eq:A_desc}
    \mathcal L_{\rho^k}(\mathbf U^{k{+}1},\mathbf V^{k{+}1},\mathbf A^{k{+}1},\textcolor{black}{\mathbf W^k}){\le}
\mathcal L_{\rho^k}(\mathbf U^{k{+}1},\mathbf V^{k{+}1},\mathbf A^k,\textcolor{black}{\mathbf W^k}).
\end{equation}
\textcolor{black}{Fix $k$ and the primal variables. 
Since
$\frac{\rho}{2}\bigl(\|r+\mathbf W\|_F^2-\|\mathbf W\|_F^2\bigr)=\rho\langle \mathbf W, r\rangle+\frac{\rho}{2}\|r\|_F^2$,
the difference of the augmented Lagrangians before and after the dual/penalty update is}
\begin{multline*}
\mathcal L_{\rho^{k+1}}(\mathbf U^{k+1},\mathbf V^{k+1},\mathbf A^{k+1},\textcolor{black}{\mathbf W^{k+1}})\\ 
- \mathcal L_{\rho^k}(\mathbf U^{k+1},\mathbf V^{k+1},\mathbf A^{k+1},\textcolor{black}{\mathbf W^k}) \\
\hspace{-0.2cm}= \sum_X \left(\textcolor{black}{\langle \rho_X^{k+1}\mathbf W_X^{k+1}-\rho_X^k\mathbf W_X^k,\; r_X^{k+1}\rangle}
+ \frac{\rho_X^{k+1}-\rho_X^k}{2}\|r_X^{k+1}\|_F^2\right).
\end{multline*}
\textcolor{black}{By the scaled dual update, we have
$\rho_X^{k+1}\mathbf W_X^{k+1}-\rho_X^k\mathbf W_X^k=\rho_X^k r_X^{k+1}$, so the inner product equals $\rho_X^k\|r_X^{k+1}\|_F^2$} and
\begin{multline} \label{eq:dual_desc}
    \mathcal L_{\rho^{k+1}}(\mathbf U^{k+1},\mathbf V^{k+1},\mathbf A^{k+1},\textcolor{black}{\mathbf W^{k+1}})\\
- \mathcal L_{\rho^k}(\mathbf U^{k+1},\mathbf V^{k+1},\mathbf A^{k+1},\textcolor{black}{\mathbf W^k})\\
= \sum_X \frac{\rho_X^{k+1}+\rho_X^k}{2}\|r_X^{k+1}\|_F^2.
\end{multline}

Combining (\ref{eq:comb_desc_l}), (\ref{eq:A_desc}), and (\ref{eq:dual_desc}) yields
\begin{multline*}
\hspace{-0.3cm}\mathcal L_{\rho^{k+1}}(\mathbf U^{k+1},\mathbf V^{k+1},\mathbf A^{k+1},\textcolor{black}{\mathbf W^{k+1}})
{\le} \mathcal L_{\rho^k}(\mathbf U^k,\mathbf V^k,\mathbf A^k,\textcolor{black}{\mathbf W^k}) \\
-c_U\|\Delta \mathbf U^k\|_F^2-c_V\|\Delta \mathbf V^k\|_F^2
+\textcolor{black}{\sum_X \frac{\rho_X^{k+1}+\rho_X^k}{2}\|r_X^{k+1}\|_F^2}.
\end{multline*}
Summing from $k=0$ to $K-1$ gives
\begin{multline*}
    \mathcal L_{\rho^K}(\mathbf U^K,\mathbf V^K,\mathbf A^K,\textcolor{black}{\mathbf W^K})-\mathcal L_{\rho^0}(\mathbf U^0,\mathbf V^0,\mathbf A^0,\textcolor{black}{\mathbf W^0})\\
\le -\sum_{k=0}^{K-1}\Bigl(c_U\|\Delta \mathbf U^k\|_F^2+c_V\|\Delta \mathbf V^k\|_F^2\Bigr)\\
+\textcolor{black}{\sum_{k=0}^{K-1}\sum_X \frac{\rho_X^{k+1}+\rho_X^k}{2}\|r_X^{k+1}\|_F^2}.
\end{multline*}
\end{proof}

\begin{proof}[e) Proof for Lemma \ref{lem:vanishing-increments}]
By Lemma~\ref{lem:alm_desc} , 
\begin{align*}
& \sum_{k=0}^{K-1} \Big(c_U \|\Delta \mathbf U^k\|^2_F + c_V \|\Delta \mathbf V^k\|^2_F\Big)  \\
& \leq\mathcal L_{\rho^0}(\mathbf U^{0},\mathbf V^{0},\mathbf A^{0},\textcolor{black}{\mathbf W^{0}}) - \mathcal L_{\rho^K}(\mathbf U^{K},\mathbf V^{K},\mathbf A^{K},\textcolor{black}{\mathbf W^{K}}) \\
& \quad +\sum_X \left( \sum_{k=0}^{K-1} \textcolor{black}{\frac{\rho_X^{k+1}+ \rho_X^k}{2}\|r_X^{k+1}\|^2_F}\right).
\end{align*}
with $0 < c_U, c_V < \infty$. 
Lemma~\ref{lem:bound_L} gives the lower boundedness of $\mathcal L_{\rho^K}$. 
By Lemma~\ref{lem:bound_y}, the penalty stops updating after index $K_X < \infty$ with $\rho_X^k = \bar \rho_X$, $\|r_X^{k+1}\|_F \to 0$, $\sum_{k = 1}^\infty (\epsilon_X^k)^2 < \infty$, and \textcolor{black}{$\{\mathbf W_X^k\}$} bounded. Since $\sum_k (\epsilon_X^k)^2 < \infty$, for $k \geq K_X$, 
\vspace{-0.25em}
\begin{multline*}
\sum_{k = K_X}^\infty \textcolor{black}{\frac{\rho_X^{k+1}+ \rho_X^k}{2}\|r_X^{k+1} \|^2_F}
= \bar \rho_X \sum_{k = K_X}^\infty \|r_X^{k+1}\|_F^2 \\
\leq \bar \rho_X \sum_{k = K_X}^\infty{\epsilon^k_X}^2 < \infty,
\end{multline*}

\vspace{-0.25em}
\textcolor{black}{while the terms with $k<K_X$ are finitely many and finite since the iterates are bounded,} and thus for $K \to \infty$, $\sum_X \left( \sum_{k=0}^{K-1} \textcolor{black}{\frac{\rho_X^{k+1}+ \rho_X^k}{2}\|r_X^{k+1}\|^2_F}\right) < \infty$.
By Lemma~\ref{lem:bound_L}, along the iterate, the augmented Lagrangian is bounded from below. Thus,
$\sum_{k = 0}^\infty \Big(c_U \|\Delta \mathbf U^k\|^2_F + c_V \|\Delta \mathbf V^k\|^2_F\Big) < \infty$. 
By $c_U, c_V > 0$, $\sum_{k = 0}^\infty \|\Delta \mathbf U^k\|^2 < \infty$
and $\sum_{k = 0}^\infty \|\Delta \mathbf V^k\|^2 < \infty$, i.e., $\|\Delta \mathbf U^k\| \to 0 $ and $\|\Delta \mathbf V^k\| \to 0$.
\end{proof}

\begin{proof}[f) Proof for Theorem \ref{prop:kkt-cluster}]
By Assumption~\ref{assum:bounded-iterates}, $\{({\mathbf{U}}^k,{\mathbf{V}}^k,{\mathbf{A}}^k)\}$ is bounded\textcolor{black}{, and by Assumption~\ref{assum:alpha_inexact}(a), so is the detection sequence $\{\boldsymbol\theta^k\}$, where $\boldsymbol\theta=\boldsymbol\alpha$ for the linear model}. By Lemma~\ref{lem:bound_y}, there exists $K_X$ such that $\forall k\ge K_X$, $\rho_X^k=\bar\rho_X$ and $\|r_X^{k+1}\|_F \to 0$ with \textcolor{black}{$\{{\mathbf{W}}_X^k\}$} bounded. By the Bolzano–Weierstrass Theorem, the whole iterate sequence admits at least one cluster point, i.e., there exists a subsequence $\{k_\ell\}$ such that
$({\mathbf{U}}^{k_\ell+1},{\mathbf{V}}^{k_\ell+1},{\mathbf{A}}^{k_\ell+1},\textcolor{black}{{\mathbf{W}}^{k_\ell+1},\boldsymbol\theta^{k_\ell+1}}) \to
({\mathbf{U}}^\star,{\mathbf{V}}^\star,{\mathbf{A}}^\star,\textcolor{black}{{\mathbf{W}}^\star,\boldsymbol\theta^\star})$ as $\ell \to \infty$. By Lemma~\ref{lem:vanishing-increments}, $\|\Delta {\mathbf{U}}^{k}\|_F\to 0$ and $\|\Delta {\mathbf{V}}^{k}\|_F\to 0$ and thus
${\mathbf{U}}^{k_\ell}\to {\mathbf{U}}^\star$, ${\mathbf{V}}^{k_\ell}\to {\mathbf{V}}^\star$. Since $r_X^{k_\ell+1} = {\mathbf{M}}_X^{k_\ell+1}-{\mathbf{A}}_X^{k_\ell+1} \to \mathbf 0$,
passing to the limit and by the continuity of
${\mathbf{U}}{\mathbf{U}}^\top$, ${\mathbf{V}}{\mathbf{V}}^\top$, and ${\mathbf{U}}{\mathbf{V}}^\top$, we obtain
${\mathbf{U}}^\star {\mathbf{U}}^{\star\top}={\mathbf{A}}_{UU}^\star$,
${\mathbf{V}}^\star {\mathbf{V}}^{\star\top}={\mathbf{A}}_{VV}^\star$, and
${\mathbf{U}}^\star {\mathbf{V}}^{\star\top}={\mathbf{A}}_{UV}^\star$.
Thus, condition \eqref{eq:primalease} holds.

\textcolor{black}{For $k\ge K_X$, the dual update gives ${\mathbf{W}}_X^{k+1}-{\mathbf{W}}_X^k=r_X^{k+1}$ and ${\mathbf{A}}_X^{k+1}-{\mathbf{A}}_X^k=({\mathbf{M}}_X^{k+1}-{\mathbf{M}}_X^k)-r_X^{k+1}+r_X^k$. Both vanish by Lemma~\ref{lem:vanishing-increments} and the continuity of ${\mathbf{M}}_X$, so ${\mathbf{A}}^{k_\ell}\to{\mathbf{A}}^\star$ and ${\mathbf{W}}^{k_\ell}\to{\mathbf{W}}^\star$ as well.} By Lemma~\ref{lem:exact_a}, ${\mathbf{A}}_X^{k+1}$ is the exact global minimizer of
\[
\min_{{\mathbf{A}}_X}\ \psi_X({\mathbf{A}}_X)+(\rho_X^k/2)\|{\mathbf{A}}_X-({\mathbf{M}}_X^{k+1}+{\mathbf{W}}_X^k)\|_F^2,
\]
where $\psi_X({\mathbf{A}}_X)=\lambda_X\|{\mathbf{A}}_X\|_{1/2}$. Hence, its optimality condition is
$\mathbf 0\in \partial \psi_X({\mathbf{A}}_X^{k+1})
+\rho_X^k\bigl({\mathbf{A}}_X^{k+1}-{\mathbf{M}}_X^{k+1}-{\mathbf{W}}_X^k\bigr)$. 

\textcolor{black}{By the dual update, $\rho_X^k\bigl({\mathbf{W}}_X^k+{\mathbf{M}}_X^{k+1}-{\mathbf{A}}_X^{k+1}\bigr)=\rho_X^{k+1}{\mathbf{W}}_X^{k+1}$, so $\rho_X^{k+1}{\mathbf{W}}_X^{k+1}\in \partial \psi_X({\mathbf{A}}_X^{k+1})$, with $\rho_X^{k_\ell+1}=\bar\rho_X$ for all sufficiently large $\ell$.}
Passing to the limit along $\{k_\ell\}$, since $\psi_X$ is finite everywhere and continuous, by the closedness of the limiting subdifferential graph, \textcolor{black}{$\bar\rho_X{\mathbf{W}}_X^\star$} $\in \lambda_X \partial \|{\mathbf{A}}_X^\star\|_{1/2}$, satisfying \eqref{eq:Y_station}.

\textcolor{black}{For the detection block, $\boldsymbol\theta^\star\in\Theta$ since $\Theta$ is closed. By the continuity of the imputation rule for missing entries w.r.t. $\boldsymbol\theta$, ${\mathbf{U}}$, and ${\mathbf{V}}$, the imputed missing entries converge along $\{k_\ell\}$ to ${\mathbf{Y}}^\star$ corresponding to $\boldsymbol\theta^\star$, ${\mathbf{U}}^\star$, and ${\mathbf{V}}^\star$.}

By Assumption~\ref{assum:err_con}, the ${\mathbf{U}}$-subproblem is solved inexactly with $\mathrm{dist}(\mathbf 0, \nabla_{\mathbf{U}} F({\mathbf{U}}^{k+1}, {\mathbf{V}}^k) + \mathcal N_{\mathbb R_+} ({\mathbf{U}}^{k+1})) \leq \varepsilon_k$. Thus, there exists $e_U^k$ with $\|e_U^k\|_F \leq \varepsilon_k$ such that $-  e_U^k \in  \nabla_{\mathbf{U}} F({\mathbf{U}}^{k+1}, {\mathbf{V}}^k) + \mathcal{N}_{\mathbb{R}_+}({\mathbf{U}}^{k+1})$, i.e.,
\[\mathbf 0 \in  \nabla_{\mathbf{U}} F({\mathbf{U}}^{k+1}, {\mathbf{V}}^k) + e_U^k + \mathcal{N}_{\mathbb{R}_+}({\mathbf{U}}^{k+1}).\]
Along $\{k_\ell\}$, ${\mathbf{U}}^{k_\ell+1} \to {\mathbf{U}}^\star$, and ${\mathbf{V}}^{k_\ell} \to {\mathbf{V}}^\star$ and $\varepsilon_{k_\ell} \to 0$, i.e., $e_U^{k_\ell} \to \mathbf 0$. By Assumption~\ref{assum:block_lipschitz}, $\nabla_{\mathbf{U}} F$ is continuous on $\mathcal D_\delta$ and
\[\nabla_{\mathbf{U}} F({\mathbf{U}}^{k_\ell+1}, {\mathbf{V}}^{k_\ell}) + e_U^{k_\ell} \to \nabla_{\mathbf{U}} F({\mathbf{U}}^\star, {\mathbf{V}}^\star).\]
By the closedness of the normal cone graph, $\mathbf{0} \in \nabla_{\mathbf{U}} F({\mathbf{U}}^\star,{\mathbf{V}}^\star) + \mathcal{N}_{\mathbb{R}_+}({\mathbf{U}}^\star)$. Since $\mathcal N_{\mathbb R_+} = \partial\iota_{\geq 0}$, $\mathbf  0 \in \nabla_{\mathbf{U}} F({\mathbf{U}}^\star, {\mathbf{V}}^\star) + \partial \iota_{\geq 0}({\mathbf{U}}^\star)$, satisfying \eqref{eq:U_station}.
Similarly, for ${\mathbf{V}}$-subproblem, passing the limit along the subsequence, $\mathbf 0 \in \nabla_{\mathbf{V}} F({\mathbf{U}}^\star, {\mathbf{V}}^\star) + \partial \iota_{\geq 0}({\mathbf{V}}^\star)$, satisfying \eqref{eq:V_station}.

\textcolor{black}{Define $\boldsymbol\Lambda^\star := {\mathbf{U}}^\star{\mathbf{V}}^{\star\top}$, so that $\lambda^{k_\ell}={\mathbf{U}}^{k_\ell}{\mathbf{V}}^{k_\ell\top}\to\boldsymbol\Lambda^\star$. By Assumption~\ref{assum:alpha_inexact}(c), there exists $\mathbf e^k$ with $\|\mathbf e^k\|\le\xi_k$ such that
$-\mathbf e^k\in\nabla\varphi_k(\boldsymbol\theta^{k+1})+\mathcal N_\Theta(\boldsymbol\theta^{k+1})$.
By Assumption~\ref{assum:alpha_inexact}(b), $g({\mathbf{Z}}_{ij};\boldsymbol\theta^\star)>0$ for all $(i,j)$ with $\mathbf{Y}_{ij}>0$, so every log-term of $\phi$ is finite near $\boldsymbol\theta^\star$. Together with Assumption~\ref{assum:alpha_inexact}(a), $\nabla_{\boldsymbol\theta}\big[\phi(g({\mathbf{Z}};\boldsymbol\theta);\boldsymbol\Lambda)+r(\boldsymbol\theta)\big]$ is continuous in $(\boldsymbol\theta,\boldsymbol\Lambda)$ and in the imputed counts near $(\boldsymbol\theta^\star,\boldsymbol\Lambda^\star)$, so
\[
\nabla\varphi_{k_\ell}(\boldsymbol\theta^{k_\ell+1})\to\nabla_{\boldsymbol\theta}\big[\phi(g({\mathbf{Z}};\boldsymbol\theta^\star);\boldsymbol\Lambda^\star)+r(\boldsymbol\theta^\star)\big].
\]
Since $\mathbf e^{k_\ell}\to\mathbf 0$, by the closedness of the graph of $\mathcal N_\Theta$,
\[
\mathbf 0\in\nabla_{\boldsymbol\theta}\big[\phi(g({\mathbf{Z}};\boldsymbol\theta^\star);\boldsymbol\Lambda^\star)+r(\boldsymbol\theta^\star)\big]+\mathcal N_\Theta(\boldsymbol\theta^\star),
\]
which is the $\boldsymbol\theta$-counterpart of \eqref{eq:kkt_alpha}--\eqref{eq:kkt_P}. For the linear model, $\boldsymbol\theta=\boldsymbol\alpha$, $r\equiv0$, and $\mathrm{vec}(g({\mathbf{Z}};\boldsymbol\alpha))={\mathbf{Z}}\boldsymbol\alpha$. With $\mathrm{vec}(\tilde{\mathbf{P}}^\star)={\mathbf{Z}}\boldsymbol\alpha^\star$, \eqref{eq:kkt_alpha} holds by definition, and $\nabla_{\boldsymbol\alpha}\phi(g({\mathbf{Z}};\boldsymbol\alpha^\star);\boldsymbol\Lambda^\star)={\mathbf{Z}}^\top\mathrm{vec}\big(\nabla\phi(\tilde{\mathbf{P}}^\star;\boldsymbol\Lambda^\star)\big)$. 
Since $\Theta$ is the preimage of a box under ${\mathbf{Z}}$, it is a polyhedron and $\mathcal N_\Theta(\boldsymbol\alpha^\star)={\mathbf{Z}}^\top\mathrm{vec}\big(\partial\iota_{[0,1]^{I\times J}}(\tilde{\mathbf{P}}^\star)\big)$, so the inclusion above is exactly \eqref{eq:kkt_alpha}.
Combining this with the established conditions for ${\mathbf{A}}$ and $({\mathbf{U}},{\mathbf{V}})$, every cluster point satisfies the KKT system \eqref{eq:U_station}--\eqref{eq:kkt_alpha}, with \eqref{eq:kkt_alpha}--\eqref{eq:kkt_P} replaced by their $\boldsymbol\theta$-counterpart for a nonlinear $g$. Hence, every cluster point is a KKT point of the split problem.}
\end{proof}

\section{Implementation Details}
\textcolor{black}{In all experiments, the proposed algorithm runs for 100 outer iterations.} The $\mathbf U/\mathbf V$-subproblems were solved by projected gradient descent with Armijo backtracking line search, with \textcolor{black}{initial step $10^{-3}$}, Armijo parameter $10^{-5}$, minimum step size $10^{-12}$, and a maximum of 3000 inner iterations. In practice, the inner loop typically converged much earlier. \textcolor{black}{The rate inside the Poisson log-likelihood is evaluated as $\max(\mathbf u_i^\top\mathbf v_j,\delta)$ with $\delta=10^{-7}$ on positive-count cells, observed or imputed.}

\paragraph*{Scale-Aware Initialization of Factors} \label{sec:init}

In nonnegative matrix factorization, the initialization can affect convergence speed, numerical stability, interpretability, and the quality of the final solution, especially when additional sparsity or structural constraints are imposed~\cite{fathi2023initialization}. Initialization matters for both direction and scale, since the local behavior of the algorithm depends on the magnitude of the initial factors~\cite{bhojanapalli2016dropping}. An improperly scaled starting point may place the iterates in an unfavorable local regime. In our model, this issue is pronounced since $\mathbf U$ and $\mathbf V$ are coupled with the detection probabilities. If the initial factors are too small, the model pushes the probabilities toward their upper bounds, leading to poor estimation of $\boldsymbol\alpha$. We adopt a scale-aware initialization that aims to place the iterates in a favorable basin of attraction.

\textcolor{black}{For the summed counts ${\mathbf{Y}}$}, $\mathbb E\left[\mathbf Y_{ij}\right]=\textcolor{black}{E\mathbf P}_{ij}\lambda_{ij}$. Assuming $\textcolor{black}{\mathbf P}_{ij}$ and $\lambda_{ij}$ weakly correlated, $\mathbb E\left[\textcolor{black}{\mathbf P}_{ij}\lambda_{ij}\right]=\mathbb E\left[\textcolor{black}{\mathbf P}_{ij}\right]\mathbb E\left[\lambda_{ij}\right]+\mathrm{Cov}\left(\textcolor{black}{\mathbf P}_{ij},\lambda_{ij}\right) \approx \mathbb E\left[\textcolor{black}{\mathbf P}_{ij}\right]\mathbb E\left[\lambda_{ij}\right]$. We approximate the detection probability by a global average level $p_0 \approx {\mathbb E[\textcolor{black}{\mathbf P}}_{ij}]$. 
Then $\mathbb E[\lambda_{ij}]\approx {{\mathbb E}[\textcolor{black}{\mathbf Y}_{i,j}}]/{(\textcolor{black}{E}p_0)}$. Since the factorization $\mathbf U\mathbf V^\top$ is scale-ambiguous, $p_0$ serves as an approximate global scaling parameter and is sufficient for a warm start. Given the surrogate matrix
$\widehat{\boldsymbol\Lambda} := {\textcolor{black}{\mathbf Y}_{ij}}/{(\textcolor{black}{E} p_0)}$, we compute a rank-$F$ 
\textcolor{black}{truncated absolute SVD of $\boldsymbol{\widehat \Lambda}$} solving
$\min_{\operatorname{rank}(\mathbf G)\le F} \|\widehat{\boldsymbol\Lambda} - \mathbf G\|_F^2$.
By the Eckart–Young–Mirsky theorem, if $\widehat{\mathbf \Lambda}$ has singular values
$\sigma_1 \ge \sigma_2 \ge \cdots$,
then the best rank-$F$ approximation in the Frobenius norm is
$\mathbf G_F = \mathbf U_F \mathbf S_F \mathbf V_F^\top$. To obtain factor matrices with approximate scales, we initialize
$\mathbf U^{0} = |\mathbf U_F| \mathbf S_F^{1/2}$ and $
\mathbf V^{0} = |\mathbf V_F| \mathbf S_F^{1/2}$ 
with the absolute values taken entrywise.

\section{Dataset Information}\label{app:data_info}
\paragraph{Plant-Pollinator (PPI) Network Dataset} This dataset records 688 pollinator species and 148 plant species in the Willamette National Forest in Oregon~\cite{pfeiffer2017plant}, \textcolor{black}{where approximately 3.5\% of all $688 \times 148$ networks are observed. Following \cite{fu2019link}, the most common species with trait information are kept and a $50 \times 50$ visitation matrix is constructed by aggregating observations spanning 2011 to 2017, with approximately 42\% of entries being nonzero.}
Information from flower survey data is used to distinguish between the remaining zero entries arising from co-occurring species that did not interact and those corresponding to species pairs that never co-occurred. The latter are treated as missing values and incorporated into the probabilistic model to differentiate between true zeros and under-counted zeros. The feature vector is constructed using pollinator energy requirements, pollinator tongue length, plant soil community, plant life form, plant flower form, and plant exclusion platform type.

\paragraph{Host-Parasite (HPI) Dataset} This dataset includes observations of 22 host species and 87 parasite species from the Sevilleta Long-Term Ecological Research Program~\cite{dallas2017predicting} with 13.5\% of all possible links observed. A $49 \times 19$ interaction matrix is constructed based on visitation counts from the most common 19 hosts and 49 species, with 22\% of the entries being non-zero. Unlike the PPI dataset, observed zeros and missing interactions are not distinguished. Side information includes host life-history traits and phylogenetic information, as well as parasite features such as life-history traits, transmission modes, genus, type, and location.

\paragraph{Location-Sonotype (Soundscape) Dataset} This dataset is collected via passive acoustic monitoring of 17 sites over 13 days in 2021 in Puerto Wilches, Santander, Colombia~\cite{Guerrero2023}. Each recording is processed using an unsupervised segmentation and clustering procedure to extract sonotypes, recurrent acoustic patterns linked to species calls~\cite{Guerrero2023}, 
\textcolor{black}{yielding a $17\times290$ site-by-sonotype count matrix. The detection features are acoustic features of sonotypes (dominant frequency, minimum and maximum vocalization frequencies, and call duration).}

\section{Additional Experiments}
\subsection{Nonlinear Detection}\label{app:nonlinear}
\textcolor{black}{We use the same setting as~\cite{ibrahim2023under}. The detection network $g(z;\boldsymbol\theta)$ has three hidden layers of 20 units and a sigmoid output on standardized features with weight decay $(w/2)\|\boldsymbol\theta\|_2^2$. We use $\tanh$ activations and full-batch L-BFGS to keep $g$ smooth. Each detection step can be stopped at a gradient-norm tolerance, following Assumption~\ref{assum:alpha_inexact}(b). 
The weight decay $w\in\{0.01,0.03,0.1,0.3\}$ is chosen on 10 separate trials by the negative log-likelihood of the model on held-out entries.}

\textcolor{black}{The factors are generated as in Section~\ref{sec:experiments} with 80\% factor sparsity and the features are ${\mathbf{Z}}_{ij}\sim\mathcal N(\mathbf 0, I_{10})$. We use the cubic sigmoid map and $\tanh$ map with a linear skip term following the implementation in \cite{ibrahim2023under}. We compare neural detection with a linear detector that includes an intercept across 30 trials.
For the cubic-sigmoid map, Table~\ref{tab:nonlinear} shows that the neural model obtains lower mean graph and detection errors. For the $\tanh$ detection, the linear model is more accurate or ties across all metrics.  On real data (Table~\ref{tab:nonlinear_real}), the linear model has slightly better performance, but the two models are within one standard deviation on every metric.
}

\begin{table}[ht]
\centering
\caption{Linear versus neural detection on synthetic data.}
\label{tab:nonlinear}
\setlength{\tabcolsep}{2pt}
\resizebox{\columnwidth}{!}{%
\begin{tabular}{lll cccc}
\toprule
truth & missing & model & MSE(${\mathbf{U}}{\mathbf{U}}^\top$) & MSE(${\mathbf{V}}{\mathbf{V}}^\top$) & MSE(${\mathbf{U}}{\mathbf{V}}^\top$) & ${\mathbf{P}}$ ($\times10^{-3}$)\\
\midrule
\multirow{4}{*}{\shortstack[l]{cubic-\\sigmoid}} & \multirow{2}{*}{5\%}  & neural & 0.18 $\pm$ 0.10 & 0.17 $\pm$ 0.09 & 0.13 $\pm$ 0.07 & 37.6 $\pm$ 13.6\\
 & & linear & 0.23 $\pm$ 0.11 & 0.21 $\pm$ 0.09 & 0.17 $\pm$ 0.07 & 52.9 $\pm$ 12.0\\
\cmidrule(lr){2-7}
 & \multirow{2}{*}{20\%} & neural & 0.29 $\pm$ 0.12 & 0.26 $\pm$ 0.10 & 0.23 $\pm$ 0.09 & 48.8 $\pm$ 20.3\\
 & & linear & 0.34 $\pm$ 0.12 & 0.30 $\pm$ 0.10 & 0.27 $\pm$ 0.10 & 54.4 $\pm$ 13.9\\
\midrule
\multirow{4}{*}{\shortstack[l]{$\tanh$}} & \multirow{2}{*}{5\%}  & neural & 0.07 $\pm$ 0.08 & 0.06 $\pm$ 0.06 & 0.03 $\pm$ 0.03 & 4.7 $\pm$ 5.7\\
 & & linear & 0.07 $\pm$ 0.06 & 0.05 $\pm$ 0.04 & 0.02 $\pm$ 0.02 & 1.0 $\pm$ 0.7\\
\cmidrule(lr){2-7}
 & \multirow{2}{*}{20\%} & neural & 0.10 $\pm$ 0.08 & 0.10 $\pm$ 0.08 & 0.06 $\pm$ 0.06 & 6.6 $\pm$ 11.8\\
 & & linear & 0.10 $\pm$ 0.08 & 0.09 $\pm$ 0.06 & 0.05 $\pm$ 0.04 & 2.0 $\pm$ 2.1\\
\bottomrule
\end{tabular}}
\vspace{-0.1cm}
\end{table}

\begin{table}[ht]
\scriptsize
\centering
\caption{Linear versus neural detection on real data.}
\label{tab:nonlinear_real}
\setlength{\tabcolsep}{2pt}
\begin{tabular}{ll ccc}
\toprule
data & model & rRMSE & AUROC & AUPRC\\
\midrule
\multirow{2}{*}{\shortstack[l]{HPI}} & neural & 5.26 $\pm$ 1.38 & 0.73 $\pm$ 0.07 & 0.49 $\pm$ 0.10\\
 & linear & 5.06 $\pm$ 1.23 & 0.73 $\pm$ 0.06 & 0.49 $\pm$ 0.10\\
\midrule
\multirow{2}{*}{\shortstack[l]{PPI}} & neural & 6.30 $\pm$ 1.74 & 0.72 $\pm$ 0.04 & 0.69 $\pm$ 0.04\\
 & linear & 5.88 $\pm$ 1.79 & 0.73 $\pm$ 0.04 & 0.70 $\pm$ 0.04\\
\midrule
\multirow{2}{*}{\shortstack[l]{Soundscape}} & neural & 2.35 $\pm$ 0.85 & 0.93 $\pm$ 0.02 & 0.994 $\pm$ 0.002\\
 & linear & 2.26 $\pm$ 0.77 & 0.93 $\pm$ 0.01 & 0.995 $\pm$ 0.002\\
\bottomrule
\end{tabular}
\vspace{-0.3cm}
\end{table}

\subsection{Sensitivity to Hyperparameters}\label{app:sensitivity}
\textcolor{black}{We vary one hyperparameter at a time from the values of Section~\ref{sec:experiments} ($F=8$, $\lambda=10^{-2}$, $\rho^0=10^{-3}$, growth factor $\gamma=1.001$) in the setting of Table~\ref{tab:mse_sim_con} with 80\% factor sparsity and 5\% missing entries, on the same 50 trials (Table~\ref{tab:sensitivity}) with at most 100 iterations. Feasibility gap is defined as $\max_X\|{\mathbf{M}}_X-{\mathbf{A}}_X\|_F/\|{\mathbf{M}}_X\|_F$ at the final iteration. From Table \ref{tab:sensitivity}, the proposed method is robust to all hyperparameters except for rank and initial penalty. An under-specified rank is the main risk: $F=4$ and $6$ raise the ${\mathbf{U}}{\mathbf{V}}^\top$ error to 0.331 and 0.131, whereas $F=12$ and $16$ leave it nearly unchanged but make ${\mathbf{U}}{\mathbf{U}}^\top$ and ${\mathbf{V}}{\mathbf{V}}^\top$ less accurate.  For the initial penalty, $\rho^0=0.1$ enforces the splitting constraints from the first iteration, increases the ${\mathbf{U}}{\mathbf{V}}^\top$ error, and drives the mean fitted detection to 0.99 (true 0.74). $\lambda=0.1$ over-penalizes the within-group graphs and leaves a 14\% feasibility gap at the iteration cap. All other settings remain within a small fraction of a standard deviation of the default and only affect how quickly the feasibility gap closes.}

\begin{table}[htbp]
\centering
\caption{\textcolor{black}{Sensitivity analysis of the proposed method.}}
\label{tab:sensitivity}
\scriptsize
\setlength{\tabcolsep}{2.5pt}
\begin{tabular}{lcccc}
\toprule
setting & MSE$({\mathbf{U}}{\mathbf{U}}^\top)$ & mse$({\mathbf{V}}{\mathbf{V}}^\top)$ & MSE$({\mathbf{U}}{\mathbf{V}}^\top)$ & feas.\\
\midrule
$F=4$  & 0.396 $\pm$ 0.083 & 0.382 $\pm$ 0.100 & 0.331 $\pm$ 0.067 & 2.3\%\\
$F=6$  & 0.187 $\pm$ 0.076 & 0.172 $\pm$ 0.072 & 0.131 $\pm$ 0.048 & 2.0\%\\
$F=12$ & 0.091 $\pm$ 0.044 & 0.061 $\pm$ 0.040 & 0.022 $\pm$ 0.018 & 1.6\%\\
$F=16$ & 0.119 $\pm$ 0.050 & 0.072 $\pm$ 0.037 & 0.027 $\pm$ 0.019 & 1.6\%\\
\midrule
$\lambda=10^{-4}$ & 0.054 $\pm$ 0.053 & 0.050 $\pm$ 0.041 & 0.024 $\pm$ 0.029 & 0.0\%\\
$\lambda=10^{-3}$ & 0.054 $\pm$ 0.054 & 0.050 $\pm$ 0.041 & 0.024 $\pm$ 0.029 & 0.2\%\\
$\lambda=10^{-1}$ & 0.091 $\pm$ 0.060 & 0.089 $\pm$ 0.046 & 0.023 $\pm$ 0.027 & 14.0\%\\
\midrule
$\rho^0=10^{-4}$ & 0.070 $\pm$ 0.065 & 0.059 $\pm$ 0.040 & 0.027 $\pm$ 0.032 & 5.0\%\\
$\rho^0=10^{-2}$ & 0.047 $\pm$ 0.046 & 0.046 $\pm$ 0.038 & 0.024 $\pm$ 0.027 & 0.3\%\\
$\rho^0=10^{-1}$ & 0.066 $\pm$ 0.093 & 0.047 $\pm$ 0.035 & 0.045 $\pm$ 0.081 & 0.0\%\\
\midrule
$\gamma=1.01$ & 0.054 $\pm$ 0.052 & 0.051 $\pm$ 0.040 & 0.022 $\pm$ 0.027 & 0.8\%\\
$\gamma=1.05$ & 0.053 $\pm$ 0.050 & 0.051 $\pm$ 0.039 & 0.024 $\pm$ 0.027 & 0.0\%\\
$\gamma=1.1$  & 0.059 $\pm$ 0.054 & 0.054 $\pm$ 0.040 & 0.027 $\pm$ 0.030 & 0.0\%\\
\bottomrule
\end{tabular}
\end{table}

\textcolor{black}{We also fit a cross-only variant ($\lambda_{UU}=\lambda_{VV}=0$) on synthetic data with sparse factors (80\% sparsity with 5-20\% data missing). As shown in Table \ref{tab:wg_syn}, removing both within-group penalties reduces the accuracy of all three graphs at every missing rate.
}

\begin{table}[htbp]
\centering
\caption{Ablation of the graph penalties on synthetic data.}
\label{tab:wg_syn}
\setlength{\tabcolsep}{1.7pt}
\resizebox{\columnwidth}{!}{%
\begin{tabular}{l ccc ccc}
\toprule
& \multicolumn{3}{c}{full} & \multicolumn{3}{c}{cross-only}\\
\cmidrule(lr){2-4}\cmidrule(lr){5-7}
missing & MSE(${\mathbf{U}}{\mathbf{U}}^\top$) & MSE(${\mathbf{V}}{\mathbf{V}}^\top$) & MSE(${\mathbf{U}}{\mathbf{V}}^\top$) & MSE(${\mathbf{U}}{\mathbf{U}}^\top$) & MSE(${\mathbf{V}}{\mathbf{V}}^\top$) & MSE(${\mathbf{U}}{\mathbf{V}}^\top$)\\
\midrule
0\% & 0.048 & 0.042 & 0.018 & 0.057 & 0.050 & 0.020\\
5\% & 0.060 & 0.054 & 0.025 & 0.074 & 0.067 & 0.030\\
20\% & 0.107 & 0.090 & 0.057 & 0.134 & 0.113 & 0.068\\
\bottomrule
\end{tabular}}
\vspace{-0.3cm}
\end{table}

\bibliographystyle{IEEEtran}

\bibliography{refs} 

@inproceedings{haeffele2014structured,
  title={Structured low-rank matrix factorization: Optimality, algorithm, and applications to image processing},
  author={Haeffele, Benjamin and Young, Eric and Vidal, Rene},
  booktitle={International conference on machine learning},
  pages={2007--2015},
  year={2014},
  organization={PMLR}
}

@article{candes2011robust,
  title={Robust principal component analysis?},
  author={Cand{\`e}s, Emmanuel J and Li, Xiaodong and Ma, Yi and Wright, John},
  journal={Journal of the ACM (JACM)},
  volume={58},
  number={3},
  pages={1--37},
  year={2011},
  publisher={ACM New York, NY, USA}
}

@article{hong2016admm,
  title   = {Convergence analysis of alternating direction method of multipliers for a family of nonconvex problems},
  author  = {Hong, Mingyi and Luo, Zhi-Quan and Razaviyayn, Meisam},
  journal = {SIAM Journal on Optimization},
  volume  = {26},
  number  = {1},
  pages   = {337--364},
  year    = {2016},
  doi     = {10.1137/140990309},
}

@inproceedings{jain2013low,
  title={Low-rank matrix completion using alternating minimization},
  author={Jain, Prateek and Netrapalli, Praneeth and Sanghavi, Sujay},
  booktitle={Proceedings of the forty-fifth annual ACM symposium on Theory of computing},
  pages={665--674},
  year={2013}
}

@inproceedings{bhojanapalli2016dropping,
  title={Dropping convexity for faster semi-definite optimization},
  author={Bhojanapalli, Srinadh and Kyrillidis, Anastasios and Sanghavi, Sujay},
  booktitle={Conference on Learning Theory},
  pages={530--582},
  year={2016},
  organization={PMLR}
}

@inproceedings{park2017non,
  title={Non-square matrix sensing without spurious local minima via the Burer-Monteiro approach},
  author={Park, Dohyung and Kyrillidis, Anastasios and Carmanis, Constantine and Sanghavi, Sujay},
  booktitle={Artificial Intelligence and Statistics},
  pages={65--74},
  year={2017},
  organization={PMLR}
}

@article{park2018finding,
  title={Finding low-rank solutions via nonconvex matrix factorization, efficiently and provably},
  author={Park, Dohyung and Kyrillidis, Anastasios and Caramanis, Constantine and Sanghavi, Sujay},
  journal={SIAM Journal on Imaging Sciences},
  volume={11},
  number={4},
  pages={2165--2204},
  year={2018},
  publisher={SIAM}
}

@inproceedings{zhang2018unified,
  title={A unified framework for nonconvex low-rank plus sparse matrix recovery},
  author={Zhang, Xiao and Wang, Lingxiao and Gu, Quanquan},
  booktitle={International Conference on Artificial Intelligence and Statistics},
  pages={1097--1107},
  year={2018},
  organization={PMLR}
}

@article{bertsimas2023sparse,
  title={Sparse plus low rank matrix decomposition: A discrete optimization approach},
  author={Bertsimas, Dimitris and Cory-Wright, Ryan and Johnson, Nicholas AG},
  journal={Journal of Machine Learning Research},
  volume={24},
  number={267},
  pages={1--51},
  year={2023}
}

@inproceedings{dubois2019fast,
  title={Fast algorithms for sparse reduced-rank regression},
  author={Dubois, Benjamin and Delmas, Jean-Fran{\c{c}}ois and Obozinski, Guillaume},
  booktitle={The 22nd international conference on artificial intelligence and statistics},
  pages={2415--2424},
  year={2019},
  organization={PMLR}
}

@article{huang2016flexible,
  title={A flexible and efficient algorithmic framework for constrained matrix and tensor factorization},
  author={Huang, Kejun and Sidiropoulos, Nicholas D and Liavas, Athanasios P},
  journal={IEEE Transactions on Signal Processing},
  volume={64},
  number={19},
  pages={5052--5065},
  year={2016},
  publisher={IEEE}
}

@inproceedings{gu2016low,
  title={Low-rank and sparse structure pursuit via alternating minimization},
  author={Gu, Quanquan and Wang, Zhaoran Wang and Liu, Han},
  booktitle={Artificial Intelligence and Statistics},
  pages={600--609},
  year={2016},
  organization={PMLR}
}

@inproceedings{shang2014recovering,
  title={Recovering low-rank and sparse matrices via robust bilateral factorization},
  author={Shang, Fanhua and Liu, Yuanyuan and Cheng, James and Cheng, Hong},
  booktitle={2014 IEEE International Conference on Data Mining},
  pages={965--970},
  year={2014},
  organization={IEEE}
}

@book{wright2022high,
  title={High-dimensional data analysis with low-dimensional models: Principles, computation, and applications},
  author={Wright, John and Ma, Yi},
  year={2022},
  publisher={Cambridge University Press}
}

@article{fu2019link,
  title={Link prediction under imperfect detection: Collaborative filtering for ecological networks},
  author={Fu, Xiao and Seo, Eugene and Clarke, Justin and Hutchinson, Rebecca A},
  journal={IEEE Transactions on Knowledge and Data Engineering},
  volume={33},
  number={8},
  pages={3117--3128},
  year={2019},
  publisher={IEEE}
}

@article{boyd2011distributed,
  title={Distributed optimization and statistical learning via the alternating direction method of multipliers},
  author={Boyd, Stephen and Parikh, Neal and Chu, Eric and Peleato, Borja and Eckstein, Jonathan and others},
  journal={Foundations and Trends{\textregistered} in Machine learning},
  volume={3},
  number={1},
  pages={1--122},
  year={2011},
  publisher={Now Publishers, Inc.}
}

@article{hong2016convergence,
  title={Convergence analysis of alternating direction method of multipliers for a family of nonconvex problems},
  author={Hong, Mingyi and Luo, Zhi-Quan and Razaviyayn, Meisam},
  journal={SIAM Journal on Optimization},
  volume={26},
  number={1},
  pages={337--364},
  year={2016},
  publisher={SIAM}
}

@inproceedings{yuanadmm,
  title={ADMM for Nonconvex Optimization under Minimal Continuity Assumption},
  author={Yuan, Ganzhao},
  booktitle={The Thirteenth International Conference on Learning Representations},
  year = {2025}
}

@article{shi2020penalty,
  title={Penalty dual decomposition method for nonsmooth nonconvex optimization—Part I: Algorithms and convergence analysis},
  author={Shi, Qingjiang and Hong, Mingyi},
  journal={IEEE Transactions on Signal Processing},
  volume={68},
  pages={4108--4122},
  year={2020},
  publisher={IEEE}
}

@article{lee2000algorithms,
  title={Algorithms for non-negative matrix factorization},
  author={Lee, Daniel and Seung, H Sebastian},
  journal={Advances in neural information processing systems},
  volume={13},
  year={2000}
}

@article{fernandez2012local,
  title={Local convergence of exact and inexact augmented Lagrangian methods under the second-order sufficient optimality condition},
  author={Fern{\'a}ndez, Dami{\'a}n and Solodov, Mikhail V},
  journal={SIAM Journal on Optimization},
  volume={22},
  number={2},
  pages={384--407},
  year={2012},
  publisher={SIAM}
}

@article{hallak2023adaptive,
  title={An adaptive Lagrangian-based scheme for nonconvex composite optimization},
  author={Hallak, Nadav and Teboulle, Marc},
  journal={Mathematics of Operations Research},
  volume={48},
  number={4},
  pages={2337--2352},
  year={2023},
  publisher={INFORMS}
}

@article{el2025convergence,
  title={Convergence rates for an inexact linearized ADMM for nonsmooth nonconvex optimization with nonlinear equality constraints},
  author={El Bourkhissi, Lahcen and Necoara, Ion},
  journal={Computational Optimization and Applications},
  pages={1--39},
  year={2025},
  publisher={Springer}
}

@article{fathi2023initialization,
  title={Initialization for non-negative matrix factorization: a comprehensive review},
  author={Fathi Hafshejani, Sajad and Moaberfard, Zahra},
  journal={International Journal of Data Science and Analytics},
  volume={16},
  number={1},
  pages={119--134},
  year={2023},
  publisher={Springer}
}

@article{haeffele2019structured,
  title={Structured low-rank matrix factorization: Global optimality, algorithms, and applications},
  author={Haeffele, Benjamin D and Vidal, Ren{\'e}},
  journal={IEEE transactions on pattern analysis and machine intelligence},
  volume={42},
  number={6},
  pages={1468--1482},
  year={2019},
  publisher={IEEE}
}

@article{joseph2009modeling,
  title={Modeling abundance using N-mixture models: the importance of considering ecological mechanisms},
  author={Joseph, Liana N and Elkin, Ch{\'e} and Martin, Tara G and Possingham, Hugh P},
  journal={Ecological Applications},
  volume={19},
  number={3},
  pages={631--642},
  year={2009},
  publisher={Wiley Online Library}
}

@article{royle2004n,
  title={N-mixture models for estimating population size from spatially replicated counts},
  author={Royle, J Andrew},
  journal={Biometrics},
  volume={60},
  number={1},
  pages={108--115},
  year={2004},
  publisher={Oxford University Press}
}

@article{kellner2014accounting,
  title={Accounting for imperfect detection in ecology: a quantitative review},
  author={Kellner, Kenneth F and Swihart, Robert K},
  journal={PloS one},
  volume={9},
  number={10},
  pages={e111436},
  year={2014},
  publisher={Public Library of Science San Francisco, USA}
}

@article{mackenzie2002estimating,
  title={Estimating site occupancy rates when detection probabilities are less than one},
  author={MacKenzie, Darryl I and Nichols, James D and Lachman, Gideon B and Droege, Sam and Andrew Royle, J and Langtimm, Catherine A},
  journal={Ecology},
  volume={83},
  number={8},
  pages={2248--2255},
  year={2002},
  publisher={Wiley Online Library}
}

@article{barker2018reliability,
  title={On the reliability of N-mixture models for count data},
  author={Barker, Richard J and Schofield, Matthew R and Link, William A and Sauer, John R},
  journal={Biometrics},
  volume={74},
  number={1},
  pages={369--377},
  year={2018},
  publisher={Wiley Online Library}
}

@article{dallas2017predicting,
  title={Predicting cryptic links in host-parasite networks},
  author={Dallas, Tad and Park, Andrew W and Drake, John M},
  journal={PLoS computational biology},
  volume={13},
  number={5},
  pages={e1005557},
  year={2017},
  publisher={Public Library of Science San Francisco, CA USA}
}

@article{terry2020finding,
  title={Finding missing links in interaction networks},
  author={Terry, J Christopher D and Lewis, Owen T},
  journal={Ecology},
  volume={101},
  number={7},
  pages={e03047},
  year={2020},
  publisher={Wiley Online Library}
}

@article{bascompte2007networks,
  title={Networks in ecology},
  author={Bascompte, Jordi},
  journal={Basic and applied ecology},
  volume={8},
  number={6},
  pages={485--490},
  year={2007},
  publisher={Elsevier}
}

@article{jordano2003invariant,
  title={Invariant properties in coevolutionary networks of plant--animal interactions},
  author={Jordano, Pedro and Bascompte, Jordi and Olesen, Jens M},
  journal={Ecology letters},
  volume={6},
  number={1},
  pages={69--81},
  year={2003},
  publisher={Wiley Online Library}
}

@article{koren2009matrix,
  title={Matrix factorization techniques for recommender systems},
  author={Koren, Yehuda and Bell, Robert and Volinsky, Chris},
  journal={Computer},
  volume={42},
  number={8},
  pages={30--37},
  year={2009},
  publisher={IEEE}
}

@article{jordano2016sampling,
  title={Sampling networks of ecological interactions},
  author={Jordano, Pedro},
  journal={Functional ecology},
  volume={30},
  number={12},
  pages={1883--1893},
  year={2016},
  publisher={Wiley Online Library}
}

@article{wang2019global,
  title={Global convergence of ADMM in nonconvex nonsmooth optimization},
  author={Wang, Yu and Yin, Wotao and Zeng, Jinshan},
  journal={Journal of Scientific Computing},
  volume={78},
  number={1},
  pages={29--63},
  year={2019},
  publisher={Springer}
}

@article{xu2012lhalf,
  title={L1/2 regularization: A thresholding representation theory and a fast solver},
  author={Xu, Zongben and Chang, Xiaojun and Xu, Fengmin and Zhang, Hai},
  journal={IEEE Transactions on Neural Networks and Learning Systems},
  volume={23},
  number={7},
  pages={1013--1027},
  year={2012},
  publisher={IEEE}
}

@article{Duarte2021,
  title={Structure and dynamics of mixed-species choruses in a tropical anuran assemblage: insights from network analysis},
  author={Duarte, Marina H. L. and Llusia, Diego and Rodrigues, Samuel S. and Nascimento, Luciana B.},
  journal={Ethology},
  volume={127},
  pages={643--650},
  year={2021},
  doi={10.1111/eth.13198}
}

@article{Dawson2026,
  title={Ecoacoustics for context-rich direct and indirect trophic interaction data and ecological network construction},
  author={Dawson, Will and Evans, Darren M. and Abrahams, Carlos and Kitson, James J. N. and Collins, Larissa and Cuff, Jordan P.},
  journal={Methods in Ecology and Evolution},
  volume={00},
  pages={1--17},
  year={2026},
  doi={10.1111/2041-210x.70288}
}

@article{Suresh2025,
  title={Bugs and bytes: Entomological biomonitoring through the integration of deep learning and molecular analysis for merged community and network analysis},
  author={Suresh, Mukilan Deivarajan and Xin, Tong and Cook, Samantha M. and Evans, Darren M.},
  journal={Agricultural and Forest Entomology},
  volume={27},
  number={1},
  pages={35--49},
  year={2025},
  doi={10.1111/afe.12667}
}

@article{Guerrero2025,
author = {Guerrero, Maria J. and Sánchez-Giraldo, Camilo and Uribe, Cesar and Martinez Arias, Victor and Isaza, Claudia},
year = {2025},
month = {05},
pages = {1255-1272},
title = {Graphical representation of landscape heterogeneity identification through unsupervised acoustic analysis},
volume = {16},
journal = {Methods in Ecology and Evolution},
doi = {10.1111/2041-210X.70041}
}

@inproceedings{Dong2016,
  title={Learning {L}aplacian Matrix in Smooth Graph Signal Representations},
  author={Dong, Xiaowen and Thanou, Dorina and Frossard, Pascal and Vandergheynst, Pierre},
  booktitle={2016 IEEE International Conference on Acoustics, Speech and Signal Processing (ICASSP)},
  pages={6160--6164},
  year={2016},
  organization={IEEE},
  doi={10.1109/ICASSP.2016.7472858}
}

@article{Mateos2019,
  title={Connecting the Dots: Identifying Network Structure via Graph Signal Processing},
  author={Mateos, Gonzalo and Segarra, Santiago and Marques, Antonio G. and Ribeiro, Alejandro},
  journal={IEEE Signal Processing Magazine},
  volume={36},
  number={3},
  pages={16--43},
  year={2019},
  doi={10.1109/MSP.2018.2890143}
}

@article{Guerrero2023,
author = {Guerrero, Maria. J. and Bedoya, Carol. L. and López, José. D. and Daza, Juan. M. and Isaza, Claudia.},
year = {2023},
pages = {1500–1514},
title = {Acoustic animal identification using unsupervised learning.},
volume = {14},
journal = {Methods in Ecology and Evolution},
doi = {10.1111/2041-210X.14103}
}

@misc{pfeiffer2017plant,
  author    = {Pfeiffer, Vera Wilder},
  title     = {Plant Pollinator data at HJ Andrews Experimental Forest, 2011 to present},
  year      = {2017},
  note = {{Environmental Data Initiative}},
  doi       = {10.6073/pasta/fdb23c02e2e9b5ed98e03f62da603045},
}

@article{yang2019inexact,
  title={Inexact block coordinate descent algorithms for nonsmooth nonconvex optimization},
  author={Yang, Yang and Pesavento, Marius and Luo, Zhi-Quan and Ottersten, Bj{\"o}rn},
  journal={IEEE Transactions on Signal Processing},
  volume={68},
  pages={947--961},
  year={2019},
  publisher={IEEE}
}

@article{shafique2022imputing,
  title={Imputing missing data in hourly traffic counts},
  author={Shafique, Muhammad Awais},
  journal={Sensors},
  volume={22},
  number={24},
  pages={9876},
  year={2022},
  publisher={MDPI}
}

@article{nakashima2022double,
  title={Double-observer approach with camera traps can correct imperfect detection and improve the accuracy of density estimation of unmarked animal populations},
  author={Nakashima, Yoshihiro and Hongo, Shun and Mizuno, Kaori and Yajima, Gota and Dzefck, Zeun’s CB},
  journal={Scientific Reports},
  volume={12},
  number={1},
  pages={2011},
  year={2022},
  publisher={Nature Publishing Group UK London}
}

@inproceedings{ibrahim2023under,
  title={Under-counted tensor completion with neural incorporation of attributes},
  author={Ibrahim, Shahana and Fu, Xiao and Hutchinson, Rebecca and Seo, Eugene},
  booktitle={International Conference on Machine Learning},
  pages={14283--14315},
  year={2023},
  organization={PMLR}
}

@article{conn1991globally,
  title={A globally convergent augmented Lagrangian algorithm for optimization with general constraints and simple bounds},
  author={Conn, Andrew R and Gould, Nicholas IM and Toint, Philippe},
  journal={SIAM Journal on Numerical Analysis},
  volume={28},
  number={2},
  pages={545--572},
  year={1991},
  publisher={SIAM}
}

@article{hoyer2004non,
  title={Non-negative matrix factorization with sparseness constraints},
  author={Hoyer, Patrik O},
  journal={Journal of machine learning research},
  volume={5},
  number={Nov},
  pages={1457--1469},
  year={2004}
}

@article{boutsidis2008svd,
  title={SVD based initialization: A head start for nonnegative matrix factorization},
  author={Boutsidis, Christos and Gallopoulos, Efstratios},
  journal={Pattern recognition},
  volume={41},
  number={4},
  pages={1350--1362},
  year={2008},
  publisher={Elsevier}
}

@inproceedings{nguyen2025under,
  title={Under-Counted Matrix Completion Without Detection Features},
  author={Nguyen, Tri and Ibrahim, Shahana and Hutchinson, Rebecca A and Fu, Xiao},
  booktitle={ICASSP 2025-2025 IEEE International Conference on Acoustics, Speech and Signal Processing (ICASSP)},
  pages={1--5},
  year={2025},
  organization={IEEE}
}

@article{gao2020learning,
  title={Learning vertex representations for bipartite networks},
  author={Gao, Ming and He, Xiangnan and Chen, Leihui and Liu, Tingting and Zhang, Jinglin and Zhou, Aoying},
  journal={IEEE transactions on knowledge and data engineering},
  volume={34},
  number={1},
  pages={379--393},
  year={2020},
  publisher={IEEE}
}

@article{liu2016neighborhood,
  title={Neighborhood regularized logistic matrix factorization for drug-target interaction prediction},
  author={Liu, Yong and Wu, Min and Miao, Chunyan and Zhao, Peilin and Li, Xiao-Li},
  journal={PLoS computational biology},
  volume={12},
  number={2},
  pages={e1004760},
  year={2016},
  publisher={Public Library of Science San Francisco, CA USA}
}

@article{melamed2014community,
  title={Community structures in bipartite networks: A dual-projection approach},
  author={Melamed, David},
  journal={PloS one},
  volume={9},
  number={5},
  pages={e97823},
  year={2014},
  publisher={Public Library of Science San Francisco, USA}
}

@article{soni2016noisy,
  title={Noisy matrix completion under sparse factor models},
  author={Soni, Akshay and Jain, Swayambhoo and Haupt, Jarvis and Gonella, Stefano},
  journal={IEEE Transactions on Information Theory},
  volume={62},
  number={6},
  pages={3636--3661},
  year={2016},
  publisher={IEEE}
}

@article{li1999mixture,
  title={Mixture density estimation},
  author={Li, Jonathan and Barron, Andrew},
  journal={Advances in neural information processing systems},
  volume={12},
  year={1999}
}

@article{huang2014nmf,
  title={Non-negative matrix factorization revisited: Uniqueness and algorithm for symmetric decomposition},
  author={Huang, Kejun and Sidiropoulos, Nicholas D and Swami, Ananthram},
  journal={IEEE Transactions on Signal Processing},
  volume={62},
  number={1},
  pages={211--224},
  year={2013},
  publisher={IEEE}
}

@article{fu2019nmf,
  author={Fu, Xiao and Huang, Kejun and Sidiropoulos, Nicholas D. and Ma, Wing-Kin},
  journal={IEEE Signal Processing Magazine}, 
  title={Nonnegative Matrix Factorization for Signal and Data Analytics: Identifiability, Algorithms, and Applications}, 
  year={2019},
  volume={36},
  number={2},
  pages={59-80},
  doi={10.1109/MSP.2018.2877582}}
\end{document}